\documentclass{article}

\usepackage[utf8]{inputenc}
\usepackage{fullpage}
\usepackage{xcolor}
\usepackage[round]{natbib}
\usepackage[colorlinks,
            linkcolor=black,
            anchorcolor=black,
            citecolor=blue]{hyperref}
\usepackage{graphicx}

\usepackage{amsmath,amsfonts,bm, amsthm}

\usepackage{thmtools,thm-restate}
\newtheorem{theorem}{Theorem}
\newtheorem{lemma}[theorem]{Lemma}
\newtheorem{corollary}[theorem]{Corollary}
\newtheorem{proposition}[theorem]{Proposition}
\newtheorem{definition}{Definition}
\newtheorem{assumption}{Assumption}

\def\1{\bm{1}}

\def\eps{{\varepsilon}}

\def\vzero{{\bm{0}}}
\def\vone{{\bm{1}}}

\def\vbeta{{\bm{\beta}}}
\def\vxi{{\bm{\xi}}}
\def\vDelta{{\bm{\Delta}}}
\def\vGamma{{\bm{\Gamma}}}

\def\ve{{\bm{e}}}

\def\vr{{\bm{r}}}

\def\vu{{\bm{u}}}
\def\vv{{\bm{v}}}
\def\vw{{\bm{w}}}
\def\vx{{\bm{x}}}
\def\vy{{\bm{y}}}

\def\mH{{\bm{H}}}
\def\mI{{\bm{I}}}

\def\mX{{\bm{X}}}

\DeclareMathAlphabet{\mathsfit}{\encodingdefault}{\sfdefault}{m}{sl}
\SetMathAlphabet{\mathsfit}{bold}{\encodingdefault}{\sfdefault}{bx}{n}

\newcommand{\E}{\mathbb{E}}

\newcommand{\R}{\mathbb{R}}

\newcommand{\norm}[1]{\left\|#1\right\|}

\usepackage{apptools}
\AtAppendix{\counterwithin{lemma}{section}}
\AtAppendix{}
\AtAppendix{}
\AtAppendix{\counterwithin{proposition}{section}}
\AtAppendix{}
\AtAppendix{}
\AtAppendix{\counterwithin{theorem}{section}}
\AtAppendix{}
\AtAppendix{}

\newcommand{\supt}[1]{{(#1)}}
\newcommand{\polylog}{\textrm{polylog}}
\newcommand{\poly}{{\rm{poly}}}

\newcommand{\diag}{\textrm{diag}}
\newcommand{\betamax}{\beta_{\rm{max}}}
\newcommand{\betamin}{\beta_{\rm{min}}}
\newcommand{\tldO}{\widetilde{O}}
\newcommand{\tldOmega}{\widetilde{\Omega}}
\newcommand{\tldT}{\widetilde{T}}
\newcommand{\tldC}{\widetilde{C}}

\newcommand{\email}[1]{\href{mailto:#1}{\color{black} \texttt{#1}}}

\title{Implicit Regularization Leads to Benign Overfitting \\for Sparse Linear Regression}
\author{
    Mo Zhou     \\Duke University\\\email{mozhou@cs.duke.edu}\and 
    Rong Ge     \\Duke University\\ \email{rongge@cs.duke.edu}
}

\begin{document}
\maketitle

\begin{abstract}
    In deep learning, often the training process finds an interpolator (a solution with 0 training loss), but the test loss is still low. This phenomenon, known as {\em benign overfitting}, is a major mystery that received a lot of recent attention. One common mechanism for benign overfitting is {\em implicit regularization}, where the training process leads to additional properties for the interpolator, often characterized by minimizing certain norms. However, even for a simple sparse linear regression problem $y = \vbeta^{*\top} \vx +\xi$ with sparse $\vbeta^*$, neither minimum $\ell_1$ or $\ell_2$ norm interpolator gives the optimal test loss. In this work, we give a different parametrization of the model which leads to a new implicit regularization effect that combines the benefit of $\ell_1$ and $\ell_2$ interpolators. We show that training our new model via gradient descent leads to an interpolator with near-optimal test loss. Our result is based on careful analysis of the training dynamics and provides another example of implicit regularization effect that goes beyond norm minimization.
\end{abstract}

\section{Introduction}

Benign overfitting \--- the phenomenon that the training loss becomes 0, but the test loss remains low \--- is a major mystery in deep learning. Recently, a long line of works \citep{belkin2019reconciling,bartlett2020benign,belkin2020two,hastie2022surprises,advani2020high,koehler2021uniform} tried to explain why interpolators (solutions with 0 training loss) can still enjoy good test loss for various models. This phenomenon is interesting and was studied extensively even for simple models of linear regression (see e.g., \citet{bartlett2020benign,tsigler2020benign,hastie2022surprises}), where data $(\vx,y)$ is generated as
\[
y = \vbeta^{*\top} \vx + \xi.
\]
Here, $\vbeta^*$ is an unknown vector that we hope to learn, $\vx$ is generated from a data distribution and $\xi$ represents the noise.

One of the major explanations for benign overfitting is {\em implicit regularization}, which suggests that the training process promotes additional properties for the interpolator that it finds. In the context of the simple linear regression, it was known that fitting the model $y = \vbeta^\top \vx$ directly by gradient descent gives the $\vbeta$ with minimum $\ell_2$ norm; while parametrizing $\vbeta$ as $\vbeta = \vw^{\odot 2}-\vu^{\odot 2}$ (here $\odot 2$ represents entry-wise square) gives the $\vbeta$ with minimum $\ell_1$ norm. 

However, for  sparse linear regression, 
implicit regularization in the form of $\ell_1$ or $\ell_2$ norm minimization does not lead to benign overfitting.
More precisely, if $\vbeta^* \in \R^d$ is an $s$-sparse vector, given $n$ samples $(\vx_i,y_i)$ for $i=1,2,...,n$ where $n\ll d$,  $\vx_i\sim N(\vzero,\mI)$ and $\xi_i\sim N(0,\sigma^2)$, one can still hope to find a parameter $\vbeta$ such that the test loss $\frac{1}{2}\E[(y-\vbeta^\top \vx)^2]$ is on the order of $\sigma^2 s\log(d/s)/n$. Neither minimum $\ell_1$ or $\ell_2$ interpolator achieves anything near this guarantee: the best $\ell_2$ norm interpolator achieves a test loss of $\Omega(\norm{\vbeta^*}^2_2)$ \citep{bartlett2020benign,hastie2022surprises} while the best $\ell_1$ norm interpolator achieves a test loss of $\Omega(\sigma^2/\log(d/n))$ \citep{chatterji2022foolish,wang2022tight}.

In the sparse regression setting, \citet{muthukumar2020harmless} showed that when the model is significantly overparametrized ($d\gg n$), it is still possible to find an interpolator with near-optimal test loss. The interpolator in \citet{muthukumar2020harmless} has to be constructed {\em explicitly} through a 2-stage process which combines $\ell_1$ and $\ell_2$ norm minimization. In this paper, we ask whether such an interpolator can be found via {\em implicit regularization} \--- by directly minimizing the loss using a new parametrization.

\subsection{Our Result and Technique}
We show that implicit regularization can indeed give near-optimal interpolators (up to polylog factors) and therefore achieve benign overfitting in the sparse regression setting: 

\begin{theorem}[Informal]
In the sparse linear regression setting with unknown $s$-sparse target $\vbeta^*$, suppose we parametrize linear function $\vbeta^\top\vx$ as
\[
\vbeta = \vv + \lambda (\vw^{\odot 2} - \vu^{\odot 2}).
\]
If $\tilde{\Omega}(s^4)\le n\le \tldO(\sqrt{d})$, with proper choice of parameters, gradient descent converges to a solution $\vbeta$ with 0 training loss and test loss
\[
\norm{\vbeta-\vbeta^*}_2=O\left(\sigma\sqrt{\frac{s\log^5(d)}{n}}\right).
\]
\end{theorem}

More formal versions of this theorem appear as Theorem~\ref{thm: main} and Corollary~\ref{cor:main}. Note that the test loss is within polylog factor to the minimax rate.

The model we use is similar to a 2-layer scalar network (which gives the $\vw^{\odot 2} - \vu^{\odot 2}$ term) with a skip-through connection (the $\vv$ term) like in the ResNet \citep{he2016deep}. 
Intuitively, the term $\lambda (\vw^{\odot 2} - \vu^{\odot 2})$ promotes minimum $\ell_1$ norm properties and can be used to fit the sparse signal $\vbeta^*$, while the term $\vv$ promotes minimum $\ell_2$ norm properties and can be used to fit the noise. 

Of course, showing that training this model via gradient descent leads to the correct trade-off between fitting the signal and noise is still challenging.
We rely on dynamics analysis and show that the term $\lambda (\vw^{\odot 2} - \vu^{\odot 2})$ first grows fast to recover the sparse signal and then the term $\vv$ grows to fit the noise. Interactions between all parameters $\vv,\vw,\vu$ makes it difficult to directly derive an accurate dynamics analysis. To address this issue, we introduce a new way to decompose $\vv$ that allows us to separate the effect of learning signal and fitting noise and leads to a better characterization of the training dynamics. See details in Section~\ref{sec:intuition} and Section~\ref{sec: pf sketch}.

\subsection{Related Works}\label{subsec: related works}

There is a long line of work trying to understand implicit regularization effect, we refer the readers to some surveys for more complete discussions \citep{bartlett2021deep,dar2021farewell,vardi2022implicit}. Here, we first summarize implicit regularization effect for interpolating linear models and their variants in regression setting. We then discuss related works for implicit regularization that are more related to training dynamic analysis instead of norm-minimizing. 

\paragraph{Min-$\ell_2$-norm interpolator}
When using linear model $\vbeta^\top\vx$ for regression $\frac{1}{2n}\sum_i(\vbeta^\top\vx_i-y_i)^2$, it is known that gradient flow/descent with 0 initialization will converge to the solution that minimizes its $\ell_2$ norm (e.g., \citet{gunasekar2018characterizing}). Recently, many papers have studied the generalization error of such min-$\ell_2$-norm interpolator in the overparametrized regime where the dimension is much larger than then number of samples: \citet{hastie2022surprises,mitra2019understanding} focused on the asymptotic regime where $d,n\to\infty$ with fixed ratio, \citet{bartlett2020benign,tsigler2020benign} gave the non-asymptotic results under certain data assumption, \citet{belkin2020two} studied Gaussian data case, and \citet{zhou2020uniform,negrea2020defense,koehler2021uniform} developed different frameworks to analyze low-norm interpolator. In particular, these results suggest that min-$\ell_2$-norm interpolator can achieve benign overfitting when the spectrum of input data covariance matrix has certain structure. On the other hand, it suffers from large test loss with isotropic features (identity covariance matrix for $\vx$).

\paragraph{Min-$\ell_1$-norm interpolator}
Going beyond the simplest linear model $\vbeta^\top\vx$, when the underlying signal is known to be sparse, one could reparametrize $\vbeta$ by $\vbeta(\vw,\vu)=\vw^{\odot L}-\vu^{\odot L}$, where $\odot L$ represents element-wise $L$-th power for integer $L \ge 2$. \citet{woodworth2020kernel,azulay2021implicit,yun2021a} showed that gradient flow with such parametrization converges to min-$\ell_1$-norm solution when using small initialization and min-$\ell_2$-norm solution when using large initialization. Researchers have studied the test loss of the min-$\ell_1$-norm interpolator in the sparse noisy linear regression: \citet{mitra2019understanding,li2021minimum} studied the asymptotic regime with $d,n\to\infty$ with fixed ratio, \citet{ju2020overfitting} focused on the Gaussian data case, and \citet{chinot_robustness_2022,koehler2021uniform} developed different frameworks and analyzed min-$\ell_1$-norm interpolator as an example. 

Lower bounds are also shown in \citet{chatterji2022foolish,wang2022tight}, which suggests that min-$\ell_1$-norm interpolator does not have good generalization performance due to its sparsity. \citet{vaskevicius2019implicit,li2021implicit} showed that gradient descent with early stopping can still achieve near-optimal test loss, but these results do not give interpolating models. 

\paragraph{Hybrid model}
\citet{muthukumar2020harmless} proposed an interpolation scheme called hybrid interpolation (Definition 5 in their paper) to achieve optimal test loss. Specifically, the hybrid interpolation is a 2-step procedure to achieve benign overfitting: (1) use any estimator to recover signal (e.g., Lasso \citep{bickel2009simultaneous}); (2) use min-$\ell_2$-norm interpolator to memorize the remaining noise. Such two-step procedure shares similarity with the learning dynamics in our analysis: our model will first recover the signal using the second-order term and then fit the noise using the linear term. Different from the hybrid interpolation scheme that requires a 2-step procedure, in our setup such learning dynamics arise naturally just by running gradient descent. 

\paragraph{Beyond norm-minimization for implicit regularization} Many of the earlier works for implicit regularization shows that the training process minimizes a certain norm (or maximizes margin with respect to a norm). The first example of implicit regularization that goes beyond norm minimization works in the setting of matrices. \citet{arora2019implicit} observed that for low-rank matrix problems the solution found does not always minimize the nuclear norm. Similar idea has also been exploited in the  full-observation matrix sensing \citep{gidel2019implicit,gissin2020the}. Later \citet{li2021towards} was able to characterize the implicit regularization effect in matrix sensing problems via a greedy-low-rank-learning dynamics. Such implicit rank regularization and dynamics analysis are also studied in tensor problems \citep{razin2021implicit,razin2022implicit,ge2021understanding} and neural networks \citep{timor2023implicit,frei2023implicit}. The implicit regularization effect in our setting can be characterized using the results in \citet{li2022implicit}, but it does not directly imply the generalization guarantee. Our result shows that dynamics analysis can be important even in the simpler sparse regression model to achieve benign overfitting.

\section{Preliminary}
\label{sec:prelim}
In this section we first introduce basic notations. Then we define the precise sparse recovery problem we are solving, and the learner model/algorithms we use. Finally we state several useful properties for the data that we will use throughout our analysis.

\paragraph{Notation}
Denote $[n]=\{1,2,\ldots,n\}$. We use bold symbols to represent vectors and matrices. For vector $\vbeta\in\R^d$, given any set $A\subseteq [d]$, let $\vbeta_A:=\sum_{i\in A}\beta_i\ve_i$ be the same as $\vbeta$ for the entries in set $A$ and 0 for other entries, where $\{\ve_i\}$ is the standard basis. We use standard big-$O$ notations $\Omega,O$ to hide constants and $\tldOmega, \tldO$ to hide constants and all logarithmic factors including $\log(d), \log(n), \log(1/\sigma)$. We will drop the time sub/superscripts when the context is clear.

\paragraph{Target function and data}
Suppose the ground-truth function is 
\[
f_*(\vx)=\vbeta^{*\top} \vx,
\]
where $\vbeta^*\in\R^{d}$ is $s$-sparse. Without loss of generality, we assume $|\beta_1^*|\ge\ldots\ge|\beta_s^*|> 0$ and $\beta^*_{s+1}=\ldots=\beta^*_d=0$. Denote $S_+:=\{i:\beta_i^*> 0\}$ be the set of positive signal entries, $S_-:=\{i:\beta_i^*<0\}$ be the set of negative signal entries, and $S:=S_+ \cup S_-=[s]$ be the set of all signal entries. 
We use $\vbeta_S:=\sum_{i:\beta^*_i\ne 0}\beta_i\ve_i$ to be the vector that is same as $\vbeta$ for the signal entries in $S$ and 0 for other entries, and $\vbeta_e:=\sum_{i:\beta^*_i= 0}\beta_i\ve_i$ to be the vector that is the same as $\vbeta$ for the non-signal entries that are not in $S$ and 0 for other entries. We similarly define $\vbeta_{S_+},\vbeta_{S_-},\vbeta_{e_+},\vbeta_{e_-}$.
Denote $\betamax:=|\beta_1^*|$ be the maximum absolute value entry of $\vbeta^*_S$ and $\betamin:=|\beta_s^*|$ be the minimum absolute value entry of $\vbeta^*_S$. We assume $\betamin,\betamax=\Theta(1)$ for simplicity. Our results can generalize to arbitrary $\betamax,\betamin$ with the cost of an additional polynomial dependency on them.

We generate $n$ training data 
$\{(\vx_i,y_i)\}_{i=1}^n$ by 
\[y=f_*(\vx)+\xi,
\]
where $\vx$ is the input data, $\xi\sim N(0,\sigma^2)$ is the label noise and $y$ is the target. Denote $\mX=[\vx_1,\vx_2,\ldots,\vx_n]^\top\in\R^{n\times d}$ as the input data matrix, $\vy = (y_1,\ldots,y_n)^\top$ as the target vector and $\vxi=(\xi_1,\ldots,\xi_n)^\top$ as the noise vector. 

\paragraph{Learner model, loss and algorithm}
To learn the target function $f_*(\vx)$, we use the following model
    \begin{equation}\label{eq: model}
    \begin{aligned}
    f_{\vu,\vw,\vv}(\vx)=(\vv+\lambda \vw^{\odot 2} - \lambda \vu^{\odot 2})^\top \vx.        
    \end{aligned}
    \end{equation}
Here $\vw^{\odot 2}:=\vw\odot\vw$ and $\vu^{\odot 2}:=\vu\odot \vu$ is the element-wise square of $\vw$ and $\vu$. In general we use $\vu\odot \vv$ to denote the element-wise product of $\vu$ and $\vv$. Our model can be viewed as a linear model $\vbeta^\top\vx$ with reparametrization $\vbeta=\vv+\lambda \vw^{\odot 2} - \lambda \vu^{\odot 2}$. Such element-wise product reparametrization $\vw^{\odot 2} - \vu^{\odot 2}$ is common in the implicit bias literature \citep{woodworth2020kernel,azulay2021implicit,yun2021a}. In the view of neural networks, the learner model can also be viewed as a 2-layer diagonal linear network with a shortcut connection \citep{he2016deep}. For simplicity of notation, denote $\vbeta=\vv+\lambda \vw^{\odot 2}-\lambda \vu^{\odot2}$. We are particular interested in the overparametrized regime $n\ll d$, where the model has the ability to overfit the data without learning the target $\vbeta^*$.

Denote residual $r_i:=f_{\vu,\vw,\vv}(\vx_i) - y_i$ for $i\in[n]$ and $\vr:=(r_1,\ldots,r_n)^\top$. 
We will use gradient descent to minimize mean-square loss, that is
\begin{align*}
    L(\vu,\vw,\vv)&:=\frac{1}{2n}\sum_{i=1}^n \left( f_{\vu,\vw,\vv}(\vx_i) - y_i \right)^2.
\end{align*}
The gradient for this loss is given below:
\begin{equation}\label{eq: grad}
\left\{
\begin{aligned}
    \vw^{(t+1)} &= \vw^{(t)} - \eta \nabla_\vw L(\vu^\supt{t},\vw^\supt{t},\vv^\supt{t})
    = \vw^{(t)} - \eta \left(\frac{1}{n}\mX^\top\vr^\supt{t}\right) \odot(2\lambda \vw^\supt{t})\\
    \vu^{(t+1)} &= \vu^{(t)} - \eta \nabla_\vu L(\vu^\supt{t},\vw^\supt{t},\vv^\supt{t})
    = \vu^{(t)} + \eta \left(\frac{1}{n}\mX^\top\vr^\supt{t}\right) \odot(2\lambda \vu^\supt{t})\\
    \vv^{(t+1)} &= \vv^{(t)} - \eta \nabla_\vv L(\vu^\supt{t},\vw^\supt{t},\vv^\supt{t})
    = \vv^{(t)} - \eta \frac{1}{n}\mX^\top\vr^\supt{t}.
\end{aligned}\right.    
\end{equation}

\paragraph{Properties the input data}

We use several key properties of the input data matrix $\mX$ and noise $\vxi$. First is the notion of Restricted Isometry Property (RIP), which is standard in the literature \citep{candes2005decoding}.
\begin{definition}[$(k,\delta)$-RIP]
A $n\times d$ matrix $\mX/\sqrt{n}$ is said to be $(k,\delta)$-RIP if for any $k$-sparse vector $\vbeta$ we have
\[
    (1-\delta)\norm{\vbeta}_2^2\le \norm{\mX\vbeta/\sqrt{n}}_2^2 \le (1+\delta)\norm{\vbeta}_2^2.
\]
\end{definition}

We will assume data matrix $\mX/\sqrt{n}$ satisfies $(s+1,\delta)$-RIP with $\delta=\tldO( (1+n/\sqrt{d})^{-1}s^{-3/2})$ and some regularity conditions on $\mX,\vxi$, as summarized in the Assumption~\ref{assump: 1} below. These conditions can be easily satisfied under some choice of $\mX,\vxi$, as shown later in Lemma~\ref{lem: rip gaussian}.
\begin{assumption}\label{assump: 1}
Input data matrix $\mX/\sqrt{n}$ satisfies $(s+1,\delta)$-RIP with $\delta\le c_\delta (1+n/\sqrt{d\log d})^{-1}s^{-3/2}\log^{-3}(d)$ where $c_\delta$ is a small enough constant, and $\mX,\vxi$ satisfy the following regularity conditions:
\begin{align*}
    &\norm{\vxi}_2=O(\sigma\sqrt{n}),\\
    &\norm{\frac{1}{n}\mX^\top\vxi}_\infty\le B_\xi:=O\left(\sigma\sqrt{\frac{\log d}{n}}\right),\\
    &\norm{\mX^\top\vxi}_2=O\left(\sigma\sqrt{dn}\right),\\
    &\norm{\frac{1}{n}\mX^\top\vbeta}_\infty=O\left(\frac{\norm{\vbeta}_2}{\sqrt{n}}\right)\text{ for any vector $\vbeta$},\\
    &(1-O(\sqrt{n/d}))d\le\lambda_{\rm{min}}(\mX \mX^\top)\le
    \lambda_{\rm{max}} (\mX \mX^\top)\le (1+O(\sqrt{n/d}))d.
\end{align*}
\end{assumption}
Note that the notation $B_\xi=O(\sigma\sqrt{\log(d)/n})$ not only is for notation simplicity, but also intuitively stands for the best error in $\ell_\infty$ that one could hope with Gaussian noise. Indeed, \citet{lounici2011oracle} showed that the minimax optimal $\ell_\infty$ test error is $\Omega(\sigma\sqrt{\log(d/s)/n})$. Later in our analysis, we show the test loss is closely related with $B_\xi$.

When each entry of data matrix $\mX$ is i.i.d. Gaussian and noise $\vxi$ is i.i.d. sampled from $N(\vzero,\sigma^2\mI)$, all the conditions above are satisfied as long as $\tldOmega(s^4)\le n\le\tldO(d/s^4)$. See Appendix~\ref{sec: appendix rip} for details. 
\begin{restatable}{lemma}{lemripgaussian}\label{lem: rip gaussian}
    Suppose $\mX$ is a Gaussian random matrix and $\vxi\sim N(\vzero,\sigma^2\mI)$. Then if $\tldOmega(s^4)\le n\le\tldO(d/s^4)$, we have Assumption~\ref{assump: 1} is satisfied with probability at least $1-1/d$.
\end{restatable}

\section{Main Result}

Our main result, formalized in the theorem below, shows that gradient descent on the learner model \eqref{eq: model} achieves benign overfitting.

\begin{restatable}[Main result]{theorem}{thmmain}\label{thm: main}
Under Assumption~\ref{assump: 1}, suppose there exists constant $C$ such that $\sigma\le C$. We train model \eqref{eq: model} with initialization $\vv^{(0)}=\vzero$, $\vw^\supt{0}=\vu^\supt{0}=\alpha\vone$ and follow the gradient descent update \eqref{eq: grad}. If 
$\tldOmega\left(s\right) \le n\le \tldO\left(\min\{d/s,d^{2/3}\}\right)$
and we choose 
$
    \lambda=
    \Theta\left(d\sigma^{-1} n^{-1}(\sqrt{\log (d)/n}+\sqrt{n/d})^{-1}\log^{-1}(n)\right),
$

$\alpha = 1/\poly(d)$, $\eta\le O(\sqrt{n/sd}/\lambda^3)$, then for every $t\ge T=O(\log(n/\alpha\eps) n/\eta d)$ with any given $\eps>0$ we have training loss $L(\vu^\supt{t},\vw^\supt{t},\vv^\supt{t})\le \eps$ and test loss 
\begin{align*}
&\norm{\vbeta^\supt{t}-\vbeta^*}_2
=O\left(\sqrt{s}\log^2(d)\left(\sigma\sqrt{\frac{\log(d)}{n}}+\sigma\sqrt{\frac{n}{d}}\right)\right).    
\end{align*}
\end{restatable}

Note that the final test error depends on $\log(1/\alpha)$. Since we choose $\alpha=1/\poly(d)$, it appears as $\log(d)$ in the final error bound. Also, the test loss does not depend on $1/\eps$, so it remains small when training loss $\eps$ is very close to 0.

For any interpolator $\vbeta$, its test loss has lower bound $\norm{\vbeta-\vbeta^*}_2=\Omega(\sigma\sqrt{s\log(d/s)/n}+\sigma\sqrt{n/d})$ \citep{muthukumar2020harmless}, where $\sigma\sqrt{n/d}$ comes from the min-$\ell_2$-norm interpolator that fits the noise. Thus, the above test loss is optimal up to $\poly(\log d, s)$ factors. The additional $\log d$, $s$ dependencies in our result (and the fact that $n$ cannot be larger than $d^{2/3}$) are due to technical difficulties in analyzing the dynamics. When $n=O(\sqrt{d\log d})$, the first term dominates the second term, and the above test loss becomes $O(\sigma\sqrt{s\log^5(d)/n})$. This is close to the minimax optimal rate $\Omega(\sigma\sqrt{s\log(d/s)/n})$ up to $\polylog(d)$ factors \citep{raskutti2011minimax}. 

For the Gaussian data case ($\vx \sim N(\vzero,\mI)$), by Lemma~\ref{lem: rip gaussian} we in addition need $n = \tilde{\Omega}(s^4)$ to satisfy Assumption~\ref{assump: 1}. This leads to the following corollary:

\begin{corollary}[Near minimax rate]\label{cor:main}
Under the setting of Theorem~\ref{thm: main} and the choice of $\lambda,\alpha,\eta$, suppose input data $\mX$ is Gaussian matrix and noise $\vxi\sim N(\vzero,\sigma^2\mI)$. If $\tldOmega(s^4)\le n\le \tldO(\sqrt{d})$, then for every $t\ge T=O(\log(n/\alpha\eps) n/\eta d)$ with any given $\eps>0$ we have training loss $L(\vu^\supt{t},\vw^\supt{t},\vv^\supt{t})\le \eps$ and test loss 
\begin{align*}
&\norm{\vbeta^\supt{t}-\vbeta^*}_2
=O\left(\sigma\sqrt{\frac{s\log^5(d)}{n}}\right),    
\end{align*}
which is near-optimal up to $\polylog(d)$ factors.
\end{corollary}

\begin{figure*}[t]
    \centering
    \includegraphics[width=0.38\textwidth]{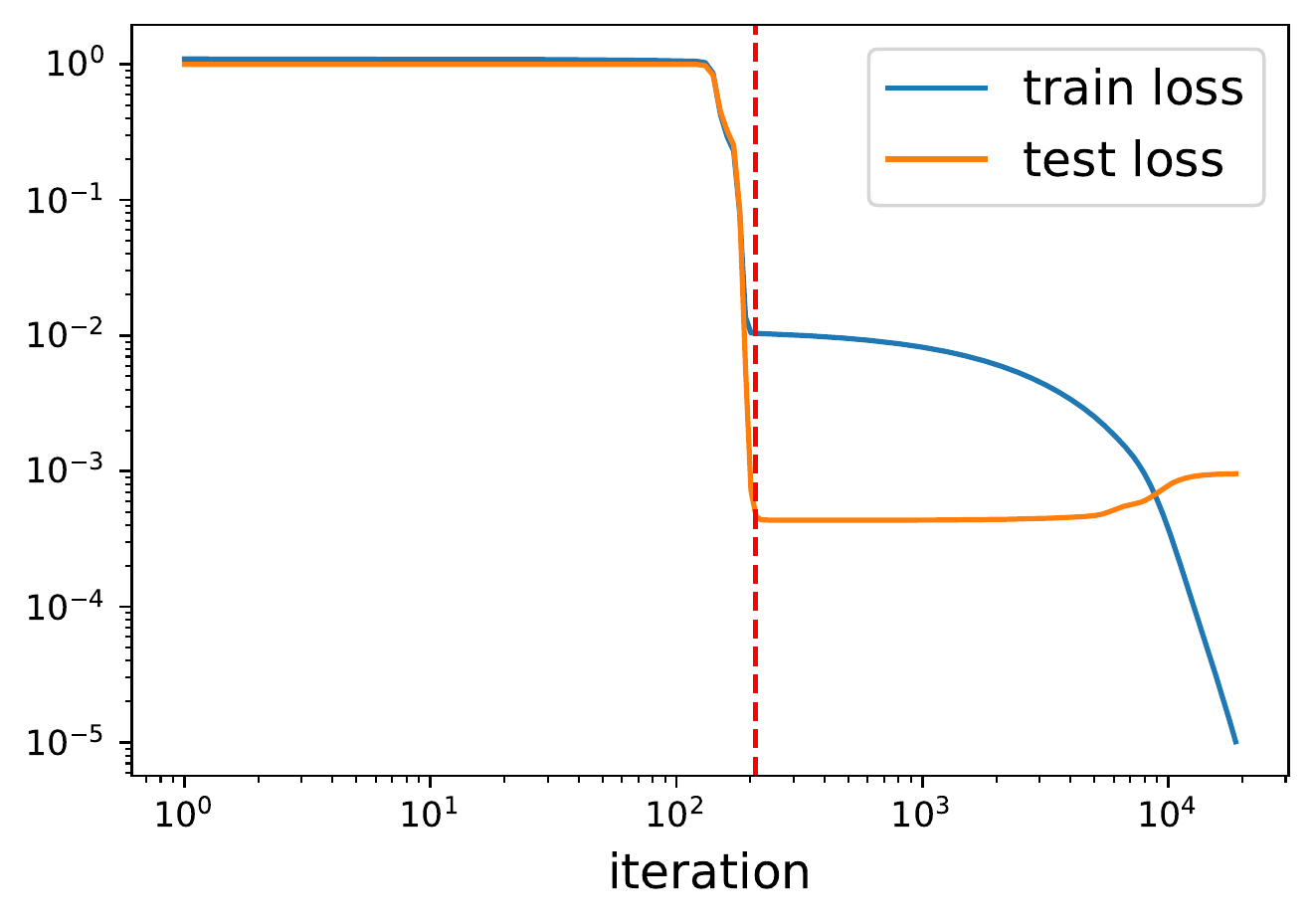}
    \qquad\qquad
    \includegraphics[width=0.45\textwidth]{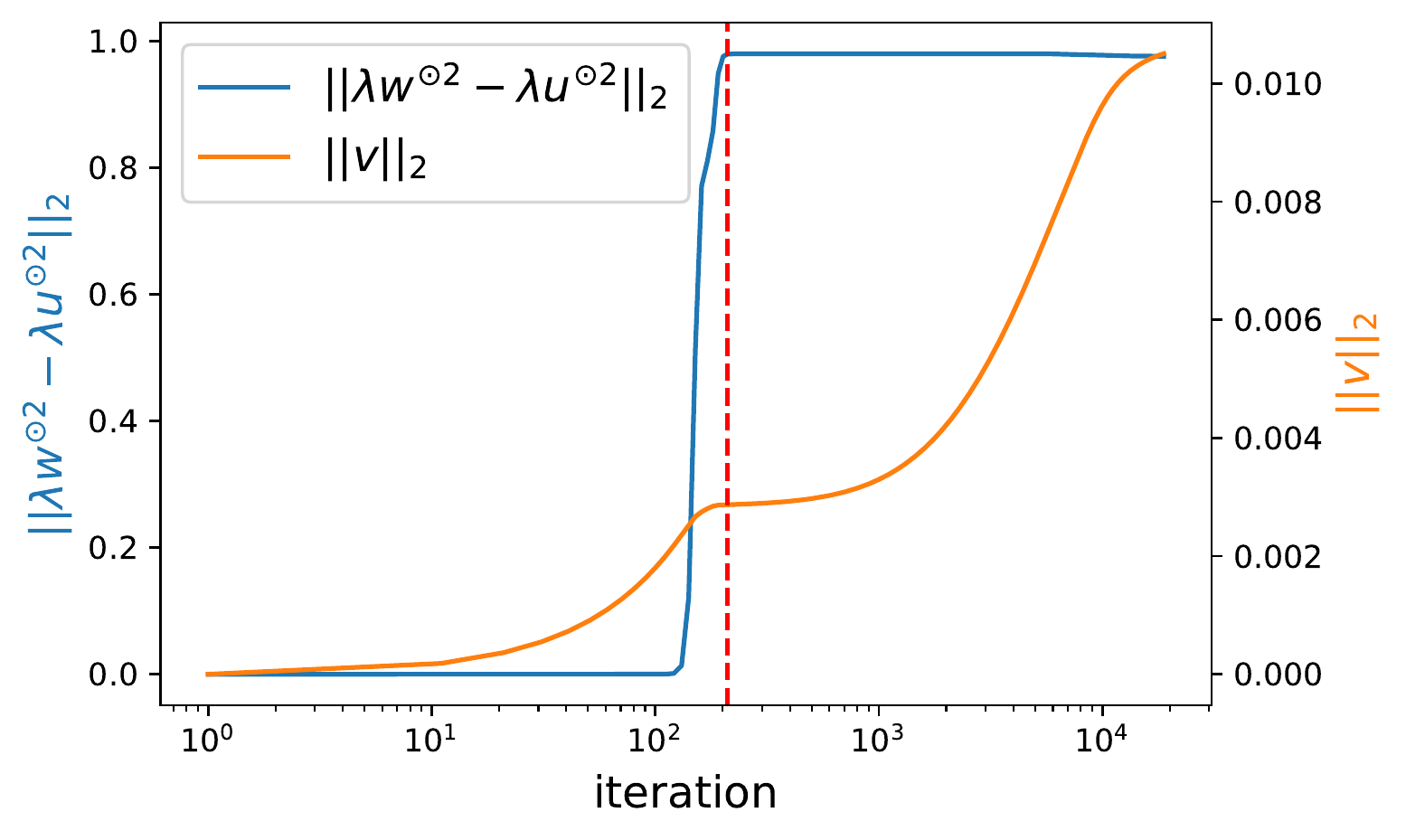}
    \vspace{0.02in}
    \caption{Training dynamics of model \eqref{eq: model} following gradient descent update \eqref{eq: grad} under $d=5\times 10^4$, $n=3\sqrt{d}$, $\sigma=0.1$, $\vbeta^*=(1/\sqrt{3},-1/\sqrt{3},1/\sqrt{3},0,\ldots,0)^\top$ and Gaussian data $\vx_i\sim N(\vzero,\mI)$. We set $\lambda = 100 d/\sigma n\log(n)(\sqrt{\log(d)/n}+\sqrt{n/d})$ and run gradient descent with $\eta=10^{-6}$ from initialization $\alpha\vone$ with $\alpha=10^{-10}$ until training loss reaches $10^{-5}$. Red vertical line stands for the transition between Stage 1 and Stage 2. 
    \textbf{Left:} training loss $L$ goes to 0 and test loss $\norm{\vbeta-\vbeta^*}_2^2$ remain small at the end.
    \textbf{Right:} norm of second-order term $\lambda(\vw^{\odot2}-\vu^{\odot2})$ grows large to recover the signal in Stage 1 and linear term $\vv$ remain small during the training. Both $x$-axis are in log scale as Stage 1 is significantly shorter than Stage 2.
    }
    \label{fig:dynamic}
\end{figure*}

\section{Intuitions for the Training Dynamics}\label{sec:intuition}
Consider the training of our model \eqref{eq: model} using gradient descent. 
 Ideally, one would hope the training process to combine the advantages of min-$\ell_1$-norm and min-$\ell_2$-norm interpolator as done explicitly in \cite{muthukumar2020harmless}: first use $\vw^{\odot2}-\vu^{\odot2}$ to learn the sparse target $\vbeta^*$ and then use $\vv$ to memorize the noise with small $\ell_2$ norm. This would require us to fix $\vv=0$ when learning the signal and fix $\vw^{\odot2}-\vu^{\odot2}$ when fitting the noise. However, since training is done on all parameters simultaneously, it's unclear why it follows this ideal dynamics.

\paragraph{Stages of training}
At a higher level, we show that the actual training dynamics of parameters $\vv,\vw,\vu$ approximately follow the above ideal dynamics in 2 stages (Figure~\ref{fig:dynamic}): \begin{itemize}
\item In Stage 1, the linear term $\vv$ remains small so that essentially the second-order term $\vw^{\odot2}-\vu^{\odot2}$ learns the signal using its bias towards sparse solution.
\item In Stage 2, $\vv$ moves to memorize the noise while $\vw^{\odot2}-\vu^{\odot2}$ roughly stays the same. Since $\vv$ is biased towards small $\ell_2$ norm, the final test loss remain small after interpolating the data. 
\end{itemize}

However, things are not as simple when we examine the dynamics carefully. It turns out that even though $\vv$ does not grow to be too large in Stage 1, it still becomes large enough so that existing analysis on $\vw$ and $\vu$ will no longer apply. To address this problem, we keep track of the dynamics of $\vv$ very carefully throughout the training process. This is done through introducing the following decompositions of $\mX^\top\mX\vv/n$ and $\vv$.

\paragraph{Decompositions of $\mX^\top\mX\vv/n$ and $\vv$}
To keep track of the dynamics of $\mX^\top\mX\vv/n$ and $\vv$, we first consider the {\em ideal} dynamics for $\vv$. We hope $\vv$ to fit the noise. If we were actually given the noise, we can use the loss function $\norm{\mX \vv-\vxi}_2^2/2n$. Running gradient descent on this function gives a trajectory for $\vv$, which can be computed explicitly. Our decomposition tries to highlight that the true trajectory of $\vv$ is close to this ideal trajectory.

There are a few more issues that we need to work with. First, for simplicity, in the ideal trajectory we approximate $\mX\mX^\top $ by $d\mI$ (which is accurate as long as $d\gg n$). Second, because of the signal, the entries of $\vv$ in $S$ may deviate significantly and in fact contribute a little bit to the fitting of the signal.

Based on these observations, we decompose both $\vv$ and $\mX^\top \mX \vv$ into three terms \--- a signal term, a noise-fitting term and an approximation error term. They are defined in the following equations:
    \begin{align}
        \frac{1}{n}\mX^\top\mX\vv^\supt{t}
        &:= \frac{d}{n} \vv_S^\supt{t} + b_t (\mX^\top\vxi)_e + \vGamma_t,\label{eq: decomp xxv}\\
        \vv^\supt{t}&:= \vv_S^\supt{t} + a_t \mX^\top\vxi + \vDelta_v^\supt{t},\label{eq: decomp v}
    \end{align}
where 
\begin{align*}
    b_{t+1}&:=b_t-\frac{\eta d}{n}\left(b_t-\frac{1}{n}\right),\\
    a_{t+1}&:=a_t - \eta \left(b_t - \frac{1}{n}\right),
\end{align*}
Here $\norm{\vGamma^\supt{t}}_\infty\le \gamma_t$, $\norm{\vDelta_v^\supt{t}}_\infty\le \zeta_t$ give $\ell_\infty$-norm bounds on the approximation error. Also recall the notation $\vbeta_S=\sum_{i:\beta^*_i\ne 0}\beta_i\ve_i$, $\vbeta_e=\sum_{i:\beta^*_i= 0}\beta_i\ve_i$. Intuitively in the decomposition of $\vv$, $\vv_S$ part tries to fit the signal, $\mX^\top\vxi$ part tries to fit the noise and the remaining term is approximation error (the decomposition of $\mX^\top \mX\vv$ has the same structure). We will show in our analysis that $\vv_S$ contributes little for learning the signal while $\mX^\top\vxi$ fits all the noise and approximation errors remain small.

The recursions of $a_t$ and $b_t$ are exactly the dynamics of $\vv$ in the ideal setting, where we fit
$\norm{\mX \vv-\vxi}_2^2/2n$ (and approximate $\mX\mX^\top$ by $d\mI$). 

Finally, the $d/n$ factor appears in front of $\vv_S$ in the decomposition of $\mX^\top\mX\vv/n$. This is because in the ideal setting (approximate $\mX\mX^\top$ by $d\mI$) the change of $\mX\mX^\top\vv/n$ is $d/n$ times larger than the change of $\vv$. One can then use simple calculations to show that the signal part $(\mX^\top\mX\vv/n)_S$ corresponds to $(d/n)\vv_S$. The non-signal part has the same factor but the $\ell_\infty$ norm there is small and hence bundled into the approximation error term.

\section{Proof Sketch}\label{sec: pf sketch}
In this section, we give the proof sketch of our main result Theorem~\ref{thm: main} with several key proof ideas. We first combine the tools we discussed in Section~\ref{sec:prelim} and  the decomposition of $\mX^\top\mX\vv/n$ and $\vv$ defined in Section~\ref{sec:intuition} to give the approximation of gradient. Then, we give the proof sketch of Stage~1 and Stage~2 in Section~\ref{subsec: stage 1} and Section~\ref{subsec: stage 2} respectively.

\paragraph{Approximation of gradient}
Given that $\mX$ is a $(s+1,\delta)$-RIP matrix, the following lemma gives useful approximation that allows us to approximate the gradient in Lemma~\ref{lem: grad approx}. The proof is a standard consequence of RIP property, which is deferred to Appendix~\ref{sec: appendix grad approx pf}.
\begin{restatable}{lemma}{lemripapprox}\label{lem: rip approx}
    Given $n\times d$ matrix $\mX/\sqrt{n}$ satisfying $(k+1,\delta)$-RIP, for any $\vbeta\in\R^d$, let 
    $
    \vDelta=\left(\frac{1}{n} \mX^\top \mX -\mI\right)\vbeta,
    $ 
    then the following hold:
    \begin{itemize}
        \item If $\vbeta$ is $k$-sparse, then $\norm{\vDelta}_\infty\le\sqrt{k}\delta\norm{\vbeta}_2$.
        \item For any vector $\vbeta$, we have $\norm{\vDelta}_\infty\le\delta\norm{\vbeta}_1$.
    \end{itemize}
\end{restatable}

The following lemma gives the approximation of the gradient. For the gradient of second-order term $\vw,\vu$, it would become the same as the gradient on the population loss $\norm{\lambda\vw^{\odot2}-\lambda\vu^{\odot2}-\vbeta^*}_2^2/2$ if $(d/n)\vv_S$ and $\vDelta_r$ are small. In particular, this suggests that the second-order term $\lambda\vw^{\odot2}-\lambda\vu^{\odot2}$ will learn the target when $\vv$ remains small.

\begin{restatable}[Gradient approximation]{lemma}{lemgradapprox}\label{lem: grad approx}
    Under Assumption~\ref{assump: 1}, we have the following gradients and their useful approximation:
    \begin{align*}
        \nabla_\vw L &= \left(\frac{1}{n} \mX^\top \vr \right)\odot(2\lambda \vw)
        =2\lambda \left(\frac{d}{n} \vv_S + \lambda \vw^{\odot 2}_{S_+} - \lambda \vu^{\odot2}_{S_-}-\vbeta^*+\vDelta_r\right)\odot \vw,\\
        \nabla_\vu L &= -\left(\frac{1}{n} \mX^\top \vr \right)\odot(2\lambda \vu)
        =-2\lambda \left(\frac{d}{n} \vv_S + \lambda \vw^{\odot 2}_{S_+} - \lambda \vu^{\odot 2}_{S_-}-\vbeta^*+\vDelta_r\right)\odot \vu,\\
        \nabla_\vv L &= \frac{1}{n} \mX^\top \vr = \frac{d}{n} \vv_S + \lambda \vw^{\odot 2}_{S_+} - \lambda \vu^{\odot 2}_{S_-}-\vbeta^*+\vDelta_r,
    \end{align*}
    where recall $S_+,S_-$ are the set of positive and negative entries of $\vbeta^*$ and $e_+=[d]\setminus S_+$, $e_-=[d]\setminus S_-$ are the corresponding complement set,
    \begin{align*}
            \norm{\vDelta_r}_\infty
            =& O\left(\left(1+\left|nb-1\right|\right)B_\xi\right)
            +\sqrt{s}\delta \frac{d}{n}\norm{\vv_S}_2
            +s\delta\norm{\frac{d}{n}\vv_S+\lambda \vw_{S_+}^{\odot 2} - \lambda \vu_{S_-}^{\odot 2}-\vbeta^*}_\infty\\
            &+O\left(\frac{d}{\sqrt{n}}\lambda\right)(\norm{\vw_{e_+}}_\infty^2+\norm{\vu_{e_-}}_\infty^2) + \gamma,
    \end{align*}
    and $b$ and $\norm{\vGamma}_\infty\le\gamma$ are defined in \eqref{eq: decomp xxv}.
\end{restatable}
Note that the factor $d/n$ in front of $\vv_S$ naturally arises when we using the decomposition of $\mX^\top\mX\vv/n$ in~\eqref{eq: decomp xxv}. This suggests that the actual part to fit the signal $\vbeta^*$ is $(d/n) \vv_S + \lambda \vw^{\odot 2}_S - \lambda \vu^{\odot2}_S$, instead of the na\"ive $\vv_S + \lambda \vw^{\odot 2}_S - \lambda \vu^{\odot2}_S$ from the form of learner model. On the other hand, since $\vv_S$ remains small, it does not affect the final test error because they are both close to $\lambda \vw^{\odot 2}_S - \lambda \vu^{\odot2}_S$.

The forms of gradients highlight the difference between the parametrization $\vv$ and $\vw^{\odot2}-\vu^{\odot2}$. For each coordinate, $w_i$ (same for $u_i$) moves according to $w_i\leftarrow (1+\eta\lambda_i)w_i$ for some growth rate $\lambda_i$, which would grow exponential fast when $\lambda_i>0$. However, the gradient for $\vv$ is not proportional to $\vv$, so it only grows linearly with time. Such difference allows us to control the order of learning dynamics ($\vv$ or $\vw^{\odot2}-\vu^{\odot2}$ grows up first). Thus, we could have the desired 2-stage learning dynamics by properly choosing the growth rate $\lambda$.

\subsection{Stage 1: Learning the Signal}\label{subsec: stage 1}
In Stage 1, our goal is to show that the linear term $\vv$ will be characterized by the decompositions \eqref{eq: decomp xxv}\eqref{eq: decomp v}, and the second-order term $\vw^{\odot2},\vu^{\odot2}$ will recover the signal $\vbeta^*$. 

The following lemma gives the ending criteria for Stage 1. We can see only the signal entries $\vw_{S_+},\vu_{S_-}$ grow large to recover $\vbeta^*$ and others such as non-signal entries $\vw_{e_+},\vu_{e_-}$ and linear term $\vv$ are remain small. Also, the loss reduces to $O(\sigma\sqrt{n})$, which is essentially the norm of noise $\norm{\vxi}_2$. The detailed proof is deferred to Appendix~\ref{sec: pf stage 1}.

\begin{restatable}[Stage 1]{lemma}{lemstageonemain}\label{lem: stage 1}
Let $C_{T_1}$ be a large enough universal constant, denote
\begin{align*}
    T_1:=\inf\Bigg\{t:
    &\norm{\frac{d}{n}\vv_S^\supt{T_1}+\lambda \vw_{S_+}^{\supt{T_1}\odot2}-\lambda \vu_{S_-}^{\supt{T_1}\odot2} - \vbeta^*}_\infty
    =C_{T_1}(B_\xi+\sigma\sqrt{n/d})\Bigg\}.
\end{align*}
Then we know $T_1=O(\log(1/\alpha)/\eta\lambda)$ and the following hold with a large enough universal constant $C_1$:
\begin{itemize}
    \item $\norm{\vw_{e_+}^\supt{T_1}}_\infty, \norm{\vu_{e_-}^\supt{T_1}}_\infty \le C_1 \alpha$.
    \item 
    $\norm{\vv_S^\supt{T_1}}_2\le C_1 \sqrt{s}(n/d)\log^2(d)(B_\xi+\sigma \sqrt{n/d})$ and $\norm{\vv^\supt{T_1}}_2\le C_1 \sigma\sqrt{n/d}$.
    \item $\norm{\vr^\supt{T_1}}_2\le C_1\sigma\sqrt{n}$.
\end{itemize}
\end{restatable}

Recall $B_\xi$ is the target infinity norm error for recovering the entries in $\vbeta^*$, when $d\gg n$, $\frac{d}{n}\vv_S+\lambda \vw_{S_+}^{\odot2}-\lambda \vu_{S_-}^{\odot2}$ achieves this error at the end of Stage 1. We focus on this term instead of $\vv_S+\lambda \vw_{S_+}^{\odot2}-\lambda \vu_{S_-}^{\odot2}$ due to its connection with the gradient shown in Lemma~\ref{lem: grad approx}. Given that $\vv_S$ is small, these two terms are in fact roughly the same.

As we discussed, a key step in the analysis is to characterize each term in the decomposition of $\mX^\top\mX\vv/n$ and $\vv$, which would imply that $\vv$ remains small in Stage 1. This is formalized in the following lemma.
\begin{lemma}[Informal]
    Consider the decomposition of $\mX^\top\mX\vv/n$ and $\vv$ in \eqref{eq: decomp xxv} \eqref{eq: decomp v}, we have for $t\le O(\log(1/\alpha/\eta\lambda))$
    \begin{align*}
        &b_t = (1-(1-\eta d/n)^t)/n\le 1/n,\\
        &a_{t} = (1-(1-\eta d/n)^t)/d\le 1/d\\
        &\norm{\vGamma_t}_\infty \le \gamma_t  = O(\sigma\sqrt{n/d}+B_\xi),\\
        &\norm{\vDelta_v}_\infty \le\zeta_{t} =O(\sigma\sqrt{n}/d).
    \end{align*}
\end{lemma}
Note that $\vv$ will memorize the noise when $b_t=1/n$ and $a_t=1/d$ as $\mX\vv^\supt{t}\approx\mX(a_t\mX^\top\vxi)\approx \vxi$. However, since $T_1=\tldO(\eta\lambda)=o(n/\eta d)$, we know $a_t=o(1/d)$ in Stage 1. This shows that $\vv$ is still small and does not yet interpolate the noise part.

Combine the above lemma with Lemma~\ref{lem: grad approx}, we have the following gradient approximation
    \begin{align*}
        \nabla_\vw L &= \left(\frac{1}{n} \mX^\top \vr \right)\odot(2\lambda \vw)
        =2\lambda (\frac{d}{n} \vv_S + \lambda \vw^{\odot 2}_S - \lambda \vu^{\odot2}_S-\vbeta^*+\vDelta_r)\odot \vw,\\
        \nabla_\vu L &= -\left(\frac{1}{n} \mX^\top \vr \right)\odot(2\lambda \vu)
        =-2\lambda (\frac{d}{n} \vv_S + \lambda \vw^{\odot 2}_S - \lambda \vu^{\odot 2}_S-\vbeta^*+\vDelta_r)\odot \vu,
    \end{align*}
    where 
    \begin{align*}
            \norm{\vDelta_r}_\infty
            =& O(B_\xi+\sigma\sqrt{n/d})
            + s\delta\norm{\frac{d}{n}\vv_S+\lambda \vw_{S_+}^{\odot 2} - \lambda \vu_{S_-}^{\odot 2}-\vbeta^*}_\infty.
    \end{align*}

Intuitively, this suggests if a coordinate of the residual $\frac{d}{n}\vv_S+\lambda\vw_S^{\odot 2}-\lambda\vu_S^{\odot 2}-\vbeta^*$ has large absolute value, then one of $\vw$ or $\vu$ will grow exponentially depending on the sign of the residual.
    
Given such gradient approximation, our goal is to show that $\vv_S$ and $\vDelta_r$ remain small so that $\vw$ and $\vu$ essentially follow the gradient on population loss $\norm{\lambda\vw^{\odot2}-\lambda\vu^{\odot2}-\vbeta^*}_2^2$/2 to recover the target $\vbeta^*$. 

In the simplest case of $s=1$, we can see that whenever the signal error $|(d/n)v_1+\lambda w_1^{2} - \lambda u_1^{2}-\beta^*_1| \ge O(B_\xi+\sigma\sqrt{n/d})$ is still large, it leads to a large gradient for either $u_1$ or $w_1$, which in turn decreases the error. 
Therefore, at the end the error will decrease to $O(B_\xi+\sigma\sqrt{n/d})$. 
In fact, due to the parameterization of $\vw^{\odot2},\vu^{\odot2}$, their growing rate would be exponential so they will grow up fast to recover the signal. 

At the same time, we can control the growth of $v_1$ by choosing a large enough $\lambda$ to ensure the length of Stage 1 $T_1$ is short. 
The non-signal entries $\vw_{e_+},\vu_{e_-}$ will also remain almost as small as their initialization, as their growth rate is much smaller compared with the signal entries. 

For higher sparsity $s$, the analysis becomes significantly more complicated because of the signal error term $\norm{\frac{d}{n}\vv_S+\lambda \vw_{S_+}^{\odot 2} - \lambda \vu_{S_-}^{\odot 2}-\vbeta^*}_\infty$ in $\norm{\vDelta_r}_\infty$. Not all the entries of  $\vbeta^*$ are of the same size, which results in different growth rates in the entries of $\vw$ and $\vu$. The entries with larger $\beta_i^*$ will be learned faster than the smaller ones, which could lead to the case where $\norm{\frac{d}{n}\vv_S+\lambda \vw_{S_+}^{\odot 2} - \lambda \vu_{S_-}^{\odot 2}-\vbeta^*}_\infty$ is much larger than the error for a particular entry $k\in S$ of $(\frac{d}{n}\vv_S+\lambda \vw_{S_+}^{\odot 2} - \lambda \vu_{S_-}^{\odot 2}-\vbeta^*)_k$. 

To deal with such issue, we show the following lemma that bound the time for reducing the signal error by half. Similar result was shown in \citet{vaskevicius2019implicit} where they do not have the linear term $\vv$. The proof relies on the observation from the gradient approximation above that the signal error will monotone decrease before reaching $\norm{\vDelta_r}_\infty$, and is made possible by the decomposition of $\vv$.

\begin{lemma}[Informal]
    Given any time $t_0$, assume $\norm{\frac{d}{n}\vv_S^\supt{t_0}+\lambda (\vw_{S_+}^\supt{t_0})^2-\lambda (\vu_{S_-}^\supt{t_0})^2 - \vbeta^*}_\infty\ge \Omega(B_\xi+\sigma\sqrt{n/d})$. Let 
    \begin{align*}
        T':=&\inf\Bigg\{t-t_0\ge 0:
        \norm{\frac{d}{n}\vv_S^\supt{t}+\lambda (\vw_{S_+}^\supt{t})^2-\lambda (\vu_{S_-}^\supt{t})^2 - \vbeta^*}_\infty
        \le \norm{\frac{d}{n}\vv_S^\supt{t_0}+\lambda (\vw_{S_+}^\supt{t_0})^2-\lambda (\vu_{S_-}^\supt{t_0})^2 - \vbeta^*}_\infty/2\Bigg\}
    \end{align*}
    be the time that signal error reduces by half.
    Then, we know $T'=O(1/\eta\lambda)$.    
\end{lemma}

Repeatedly using the above lemma, we know it takes $T_1=\tldO(1/\eta\lambda)$ time to reach the desired accuracy. Other claims follow directly.  
Detailed proofs are deferred to Appendix~\ref{sec: pf stage 1}.

\subsection{Stage 2: Memorizing the Noise}\label{sec: stage 2}\label{subsec: stage 2}
Given that in Stage 1 we know $\lambda\vw^{\odot2}-\lambda\vu^{\odot2}$ has already recovered signal $\vbeta^*$, in Stage 2 we show that the remaining noise will be memorized by the linear term $\vv$ without increasing the test loss by too much. This allows us to recover the ground-truth $\vbeta^*$ despite interpolating the data to $\eps$ training error, as formalized in the following lemma. The proof is deferred to Appendix~\ref{sec: pf stage 2}.
\begin{lemma}[Stage 2]\label{lem: stage 2}
Let $T_{2}:=\inf\{t\ge 0:L(\vw^\supt{t},\vu^\supt{t},\vv^\supt{t})\le\eps\}$. Then, we have the length of Stage~2 is $T_{2}-T_1=O((n/\eta d)\log(n/\eps))$ and the following hold for every $t\ge T_2$ with large enough universal constant $C_2$:
\begin{itemize}
    \item $\norm{\frac{d}{n}\vv_S^\supt{t}+\lambda \vw_{S_+}^{\supt{t}\odot2}-\lambda\vu_{S_-}^{\supt{t}\odot2}-\vbeta^*}_\infty\le C_2(B_\xi+\sigma\sqrt{n/d})$
    \item $\norm{\vw_{e_+}^\supt{t}}_\infty, \norm{\vu_{e_-}^\supt{t}}_\infty \le C_2\alpha$.
    \item $\norm{\vv_S^\supt{t}}_2\le C_2\sqrt{s}(n/d)\log^2(d)(B_\xi+\sigma\sqrt{n/d})$ and $\norm{\vv^\supt{t}}_2\le C_2\sigma \sqrt{n/d}$.
\end{itemize}    
\end{lemma}

Similar as in Stage 1, we still need to characterize each term in the decomposition of $\mX^\top\mX\vv/n$ and $\vv$.
\begin{lemma}[Informal]\label{lem: stage 2 decomp informal}
    Consider the decomposition of $\mX^\top\mX\vv/n$ and $\vv$ in \eqref{eq: decomp xxv} \eqref{eq: decomp v}, we have for $t\le O((n/\eta d)\log(n/\eps))$
    \begin{align*}
        &b_t = (1-(1-\eta d/n)^t)/n\le 1/n,\\
        &a_{t} = (1-(1-\eta d/n)^t)/d\le 1/d\\
        &\norm{\vGamma_t}_\infty \le \gamma_t  = O(\sigma\sqrt{n/d}+B_\xi),\\
        &\norm{\vDelta_v}_\infty \le\zeta_{t} =O((B_\xi+\sigma\sqrt{n/d})n\log(n)/d).
    \end{align*}
\end{lemma}
Unlike in Stage 1, the signal has mostly been fitted in Stage~2. This makes the gradient smaller and the time it takes for Stage 2 $\left(T_{2}-T_1=O((n/\eta d)\log(n/\eps))\right)$ is much longer than Stage 1.
Because of this longer time, we now have $b_t\approx 1/n$, $a_t\approx 1/d$ at the end of Stage~2. This implies that we essentially use linear term $\vv$ to interpolate the noise as $\mX\vv^\supt{t}\approx\mX(a_t\mX^\top\vxi)\approx\vxi$.

In the analysis of Stage 2, we have two major goals that are closely related: first, we want non-signal entries of $\vw$, $\vu$ to stay small; second, we want the residual $\norm{\vr}_2$ to decrease exponentially.

For $\vw$, $\vu$, combine the above lemma with Lemma~\ref{lem: grad approx}, we know
    \begin{align*}
        \nabla_\vw L &= \left(\frac{1}{n} \mX^\top \vr \right)\odot(2\lambda \vw),\\
        \nabla_\vu L &= -\left(\frac{1}{n} \mX^\top \vr \right)\odot(2\lambda \vu),
    \end{align*}
    where 
    \begin{align*}
        \norm{\frac{1}{n}\mX^\top\vr}_\infty 
        &= \norm{\frac{d}{n} \vv_S + \lambda \vw^{\odot 2}_S - \lambda \vu^{\odot 2}_S-\vbeta^*+\vDelta_r}_\infty
        =O(B_\xi+\sigma\sqrt{n/d}).
    \end{align*}

The infinity norm bound on $\frac{1}{n}\mX^\top\vr$ follows from a case analysis for signal and non-signal entries.
For the signal entries, using the above gradient approximation similar as in Stage 1, we can show that the signal error $\norm{\frac{d}{n} \vv_S + \lambda \vw^{\odot 2}_S - \lambda \vu^{\odot 2}_S-\vbeta^*+\vDelta_r}_\infty$ remains  $O(B_\xi+\sigma\sqrt{n/d})$.
For the non-signal entries $\vw_{e_+},\vu_{e_-}$, we know its exponential growth rate is $O(\lambda(B_\xi+\sigma\sqrt{n/d}))$ from the gradient approximation. 

The bound on $\norm{\frac{1}{n}\mX^\top\vr}_\infty$ limits the movement of $\vu$ and $\vw$. As long as 
$O(\eta\lambda(B_\xi+\sigma\sqrt{n/d})(T_2-T_1))<1$, the non-signal part of $\vu$ and $\vw$ will remain small. 

On the other hand, for the decrease rate of $\norm{\vr}_2$, the standard approach is to use ideas from Neural Tangent Kernel (NTK) \citep{jacot2018neural,du2018gradient,allen2019convergence}, and approximate the dynamics of $\vr$ as $\vr^{(t+1)} = (\mI-\eta \mH^{(t)}) \vr^{(t)}$ where $\mH^{(t)}$ is the neural tangent kernel. The decreasing rate of $\norm{\vr}_2$ can then be bounded by lowerbounding the minimum eigenvalue of $\mH^{(t)}$. However, bounding $\mH^{(t)}$ na\"ively by its distance to some initial $\mH^{(t)}$ does not work in our case. 

To fix this problem, we again rely on the dynamics of $\vv$. Lemma~\ref{lem: stage 2 decomp informal} suggests that $\vv^\supt{t}$ gets close to $\mX^\top \vxi/d$ with a rate of $\Omega(d/n)$ as $\vv^\supt{t}\approx a_t\mX^\top\vxi$ (this can also be viewed as the minimum eigenvalue of the NTK kernel restricted to $\vv$). 
This convergence rate gives a bound on the length of $T_2-T_1$, which allow us to choose an appropriate $\lambda$ to keep $\vw_{e_+},\vu_{e_-}$ small.

Once we have the bounds for the convergence rate and non-signal entries of $\vu$,$\vw$, other claims follow directly. 
Details are deferred to Appendix~\ref{sec: pf stage 2}.

Note that in the argument above, since the length of Stage~2 $T_2-T_1$ is proportional to $\log(1/\eps)$, it cannot be used when $\epsilon$ is very close to 0 as $\lambda$ is proportional to $1/(T_2-T_1)$ and would become very small. In fact, we can get rid of the dependency on $\log(1/\eps)$ with a more careful analysis. In the actual proof, we have two sub-stages for Stage 2, which uses different ways to bound the growth rate $\norm{\mX^\top\vr/n}_\infty$. For Stage 2.1, we use the argument above until $\norm{\vr}_2=O(\sigma)$. 
In Stage 2.2, given the training loss is already small enough, we use a NTK-type analysis to bound $\norm{\mX^\top\vr/n}_\infty=(1-\Omega(\eta d/n))^{t-T_1}O(\sigma/\sqrt{n})$ as $\norm{\vr}_2$ decreases with rate $\Omega(d/n)$. See Appendix~\ref{sec: pf stage 2} for details.

\section{Conclusion}

In this paper, we give a new parametrization for the sparse linear regression problem, and showed that gradient descent for this new parametrization can learn an interpolator with near-optimal test loss. This highlights the importance of choosing the correct parametrization, especially the role of linear terms in fitting noise. Our proof is based on a new dynamic analysis that shows it is possible to first learn the features, and then fit the noise using an NTK-like process. We suspect similar training dynamics may apply to more complicated problems such as low-rank matrix factorization or neural networks. 

\section*{Acknowledgement}
This work is supported by NSF Award DMS-2031849, CCF-1845171 (CAREER), CCF-1934964 (Tripods) and a Sloan Research Fellowship.

\bibliographystyle{plainnat}
\bibliography{ref}

\begin{thebibliography}{46}
\providecommand{\natexlab}[1]{#1}
\providecommand{\url}[1]{\texttt{#1}}
\expandafter\ifx\csname urlstyle\endcsname\relax
  \providecommand{\doi}[1]{doi: #1}\else
  \providecommand{\doi}{doi: \begingroup \urlstyle{rm}\Url}\fi

\bibitem[Advani et~al.(2020)Advani, Saxe, and Sompolinsky]{advani2020high}
Madhu~S Advani, Andrew~M Saxe, and Haim Sompolinsky.
\newblock High-dimensional dynamics of generalization error in neural networks.
\newblock \emph{Neural Networks}, 132:\penalty0 428--446, 2020.

\bibitem[Allen-Zhu et~al.(2019)Allen-Zhu, Li, and Song]{allen2019convergence}
Zeyuan Allen-Zhu, Yuanzhi Li, and Zhao Song.
\newblock A convergence theory for deep learning via over-parameterization.
\newblock In \emph{International Conference on Machine Learning}, pages
  242--252. PMLR, 2019.

\bibitem[Arora et~al.(2019)Arora, Cohen, Hu, and Luo]{arora2019implicit}
Sanjeev Arora, Nadav Cohen, Wei Hu, and Yuping Luo.
\newblock Implicit regularization in deep matrix factorization.
\newblock \emph{Advances in Neural Information Processing Systems}, 32, 2019.

\bibitem[Azulay et~al.(2021)Azulay, Moroshko, Nacson, Woodworth, Srebro,
  Globerson, and Soudry]{azulay2021implicit}
Shahar Azulay, Edward Moroshko, Mor~Shpigel Nacson, Blake~E Woodworth, Nathan
  Srebro, Amir Globerson, and Daniel Soudry.
\newblock On the implicit bias of initialization shape: Beyond infinitesimal
  mirror descent.
\newblock In \emph{International Conference on Machine Learning}, pages
  468--477. PMLR, 2021.

\bibitem[Baraniuk et~al.(2008)Baraniuk, Davenport, DeVore, and
  Wakin]{baraniuk2008simple}
Richard Baraniuk, Mark Davenport, Ronald DeVore, and Michael Wakin.
\newblock A simple proof of the restricted isometry property for random
  matrices.
\newblock \emph{Constructive Approximation}, 28\penalty0 (3):\penalty0
  253--263, 2008.

\bibitem[Bartlett et~al.(2020)Bartlett, Long, Lugosi, and
  Tsigler]{bartlett2020benign}
Peter~L Bartlett, Philip~M Long, G{\'a}bor Lugosi, and Alexander Tsigler.
\newblock Benign overfitting in linear regression.
\newblock \emph{Proceedings of the National Academy of Sciences}, 117\penalty0
  (48):\penalty0 30063--30070, 2020.

\bibitem[Bartlett et~al.(2021)Bartlett, Montanari, and
  Rakhlin]{bartlett2021deep}
Peter~L Bartlett, Andrea Montanari, and Alexander Rakhlin.
\newblock Deep learning: a statistical viewpoint.
\newblock \emph{Acta numerica}, 30:\penalty0 87--201, 2021.

\bibitem[Belkin et~al.(2019)Belkin, Hsu, Ma, and Mandal]{belkin2019reconciling}
Mikhail Belkin, Daniel Hsu, Siyuan Ma, and Soumik Mandal.
\newblock Reconciling modern machine-learning practice and the classical
  bias--variance trade-off.
\newblock \emph{Proceedings of the National Academy of Sciences}, 116\penalty0
  (32):\penalty0 15849--15854, 2019.

\bibitem[Belkin et~al.(2020)Belkin, Hsu, and Xu]{belkin2020two}
Mikhail Belkin, Daniel Hsu, and Ji~Xu.
\newblock Two models of double descent for weak features.
\newblock \emph{SIAM Journal on Mathematics of Data Science}, 2\penalty0
  (4):\penalty0 1167--1180, 2020.

\bibitem[Bickel et~al.(2009)Bickel, Ritov, and
  Tsybakov]{bickel2009simultaneous}
Peter~J Bickel, Ya’acov Ritov, and Alexandre~B Tsybakov.
\newblock Simultaneous analysis of lasso and dantzig selector.
\newblock \emph{The Annals of statistics}, 37\penalty0 (4):\penalty0
  1705--1732, 2009.

\bibitem[Candes and Tao(2005)]{candes2005decoding}
Emmanuel~J Candes and Terence Tao.
\newblock Decoding by linear programming.
\newblock \emph{IEEE transactions on information theory}, 51\penalty0
  (12):\penalty0 4203--4215, 2005.

\bibitem[Chatterji and Long(2022)]{chatterji2022foolish}
Niladri~S Chatterji and Philip~M Long.
\newblock Foolish crowds support benign overfitting.
\newblock \emph{Journal of Machine Learning Research}, 23\penalty0
  (125):\penalty0 1--12, 2022.

\bibitem[Chinot et~al.(2022)Chinot, Löffler, and van~de
  Geer]{chinot_robustness_2022}
Geoffrey Chinot, Matthias Löffler, and Sara van~de Geer.
\newblock On the robustness of minimum norm interpolators and regularized
  empirical risk minimizers.
\newblock \emph{The Annals of Statistics}, 50\penalty0 (4), August 2022.
\newblock ISSN 0090-5364.
\newblock \doi{10.1214/22-AOS2190}.

\bibitem[Dar et~al.(2021)Dar, Muthukumar, and Baraniuk]{dar2021farewell}
Yehuda Dar, Vidya Muthukumar, and Richard~G Baraniuk.
\newblock A farewell to the bias-variance tradeoff? an overview of the theory
  of overparameterized machine learning.
\newblock \emph{arXiv preprint arXiv:2109.02355}, 2021.

\bibitem[Du et~al.(2019)Du, Zhai, Poczos, and Singh]{du2018gradient}
Simon~S. Du, Xiyu Zhai, Barnabas Poczos, and Aarti Singh.
\newblock Gradient descent provably optimizes over-parameterized neural
  networks.
\newblock In \emph{International Conference on Learning Representations}, 2019.
\newblock URL \url{https://openreview.net/forum?id=S1eK3i09YQ}.

\bibitem[Frei et~al.(2023)Frei, Vardi, Bartlett, Srebro, and
  Hu]{frei2023implicit}
Spencer Frei, Gal Vardi, Peter Bartlett, Nathan Srebro, and Wei Hu.
\newblock Implicit bias in leaky relu networks trained on high-dimensional
  data.
\newblock In \emph{The Eleventh International Conference on Learning
  Representations}, 2023.
\newblock URL \url{https://openreview.net/forum?id=JpbLyEI5EwW}.

\bibitem[Ge et~al.(2021)Ge, Ren, Wang, and Zhou]{ge2021understanding}
Rong Ge, Yunwei Ren, Xiang Wang, and Mo~Zhou.
\newblock Understanding deflation process in over-parametrized tensor
  decomposition.
\newblock \emph{Advances in Neural Information Processing Systems},
  34:\penalty0 1299--1311, 2021.

\bibitem[Gidel et~al.(2019)Gidel, Bach, and Lacoste-Julien]{gidel2019implicit}
Gauthier Gidel, Francis Bach, and Simon Lacoste-Julien.
\newblock Implicit regularization of discrete gradient dynamics in linear
  neural networks.
\newblock \emph{Advances in Neural Information Processing Systems}, 32, 2019.

\bibitem[Gissin et~al.(2020)Gissin, Shalev-Shwartz, and Daniely]{gissin2020the}
Daniel Gissin, Shai Shalev-Shwartz, and Amit Daniely.
\newblock The implicit bias of depth: How incremental learning drives
  generalization.
\newblock In \emph{International Conference on Learning Representations}, 2020.
\newblock URL \url{https://openreview.net/forum?id=H1lj0nNFwB}.

\bibitem[Gunasekar et~al.(2018)Gunasekar, Lee, Soudry, and
  Srebro]{gunasekar2018characterizing}
Suriya Gunasekar, Jason Lee, Daniel Soudry, and Nathan Srebro.
\newblock Characterizing implicit bias in terms of optimization geometry.
\newblock In \emph{International Conference on Machine Learning}, pages
  1832--1841. PMLR, 2018.

\bibitem[Hastie et~al.(2022)Hastie, Montanari, Rosset, and
  Tibshirani]{hastie2022surprises}
Trevor Hastie, Andrea Montanari, Saharon Rosset, and Ryan~J Tibshirani.
\newblock Surprises in high-dimensional ridgeless least squares interpolation.
\newblock \emph{The Annals of Statistics}, 50\penalty0 (2):\penalty0 949--986,
  2022.

\bibitem[He et~al.(2016)He, Zhang, Ren, and Sun]{he2016deep}
Kaiming He, Xiangyu Zhang, Shaoqing Ren, and Jian Sun.
\newblock Deep residual learning for image recognition.
\newblock In \emph{Proceedings of the IEEE conference on computer vision and
  pattern recognition}, pages 770--778, 2016.

\bibitem[Jacot et~al.(2018)Jacot, Gabriel, and Hongler]{jacot2018neural}
Arthur Jacot, Franck Gabriel, and Cl{\'e}ment Hongler.
\newblock Neural tangent kernel: Convergence and generalization in neural
  networks.
\newblock \emph{Advances in neural information processing systems}, 31, 2018.

\bibitem[Ju et~al.(2020)Ju, Lin, and Liu]{ju2020overfitting}
Peizhong Ju, Xiaojun Lin, and Jia Liu.
\newblock Overfitting can be harmless for basis pursuit, but only to a degree.
\newblock \emph{Advances in Neural Information Processing Systems},
  33:\penalty0 7956--7967, 2020.

\bibitem[Koehler et~al.(2021)Koehler, Zhou, Sutherland, and
  Srebro]{koehler2021uniform}
Frederic Koehler, Lijia Zhou, Danica~J Sutherland, and Nathan Srebro.
\newblock Uniform convergence of interpolators: Gaussian width, norm bounds and
  benign overfitting.
\newblock \emph{Advances in Neural Information Processing Systems},
  34:\penalty0 20657--20668, 2021.

\bibitem[Li et~al.(2021{\natexlab{a}})Li, Nguyen, Hegde, and
  Wong]{li2021implicit}
Jiangyuan Li, Thanh Nguyen, Chinmay Hegde, and Ka~Wai Wong.
\newblock Implicit sparse regularization: The impact of depth and early
  stopping.
\newblock \emph{Advances in Neural Information Processing Systems},
  34:\penalty0 28298--28309, 2021{\natexlab{a}}.

\bibitem[Li and Wei(2021)]{li2021minimum}
Yue Li and Yuting Wei.
\newblock Minimum {$\ell_1$}-norm interpolators: Precise asymptotics and
  multiple descent.
\newblock \emph{arXiv preprint arXiv:2110.09502}, 2021.

\bibitem[Li et~al.(2021{\natexlab{b}})Li, Luo, and Lyu]{li2021towards}
Zhiyuan Li, Yuping Luo, and Kaifeng Lyu.
\newblock Towards resolving the implicit bias of gradient descent for matrix
  factorization: Greedy low-rank learning.
\newblock In \emph{International Conference on Learning Representations},
  2021{\natexlab{b}}.
\newblock URL \url{https://openreview.net/forum?id=AHOs7Sm5H7R}.

\bibitem[Li et~al.(2022)Li, Wang, Lee, and Arora]{li2022implicit}
Zhiyuan Li, Tianhao Wang, Jason~D Lee, and Sanjeev Arora.
\newblock Implicit bias of gradient descent on reparametrized models: On
  equivalence to mirror descent.
\newblock \emph{Advances in Neural Information Processing Systems},
  35:\penalty0 34626--34640, 2022.

\bibitem[Lounici et~al.(2011)Lounici, Pontil, Van De~Geer, and
  Tsybakov]{lounici2011oracle}
Karim Lounici, Massimiliano Pontil, Sara Van De~Geer, and Alexandre~B Tsybakov.
\newblock Oracle inequalities and optimal inference under group sparsity.
\newblock \emph{The annals of statistics}, 39\penalty0 (4):\penalty0
  2164--2204, 2011.

\bibitem[Mitra(2019)]{mitra2019understanding}
Partha~P Mitra.
\newblock Understanding overfitting peaks in generalization error: Analytical
  risk curves for {$\ell_2$} and {$\ell_1$} penalized interpolation.
\newblock \emph{arXiv preprint arXiv:1906.03667}, 2019.

\bibitem[Muthukumar et~al.(2020)Muthukumar, Vodrahalli, Subramanian, and
  Sahai]{muthukumar2020harmless}
Vidya Muthukumar, Kailas Vodrahalli, Vignesh Subramanian, and Anant Sahai.
\newblock Harmless interpolation of noisy data in regression.
\newblock \emph{IEEE Journal on Selected Areas in Information Theory},
  1\penalty0 (1):\penalty0 67--83, 2020.

\bibitem[Negrea et~al.(2020)Negrea, Dziugaite, and Roy]{negrea2020defense}
Jeffrey Negrea, Gintare~Karolina Dziugaite, and Daniel Roy.
\newblock In defense of uniform convergence: Generalization via derandomization
  with an application to interpolating predictors.
\newblock In \emph{International Conference on Machine Learning}, pages
  7263--7272. PMLR, 2020.

\bibitem[Raskutti et~al.(2011)Raskutti, Wainwright, and
  Yu]{raskutti2011minimax}
Garvesh Raskutti, Martin~J Wainwright, and Bin Yu.
\newblock Minimax rates of estimation for high-dimensional linear regression
  over {$\ell_q$}-balls.
\newblock \emph{IEEE transactions on information theory}, 57\penalty0
  (10):\penalty0 6976--6994, 2011.

\bibitem[Razin et~al.(2021)Razin, Maman, and Cohen]{razin2021implicit}
Noam Razin, Asaf Maman, and Nadav Cohen.
\newblock Implicit regularization in tensor factorization.
\newblock In \emph{International Conference on Machine Learning}, pages
  8913--8924. PMLR, 2021.

\bibitem[Razin et~al.(2022)Razin, Maman, and Cohen]{razin2022implicit}
Noam Razin, Asaf Maman, and Nadav Cohen.
\newblock Implicit regularization in hierarchical tensor factorization and deep
  convolutional neural networks.
\newblock In \emph{International Conference on Machine Learning}, pages
  18422--18462. PMLR, 2022.

\bibitem[Timor et~al.(2023)Timor, Vardi, and Shamir]{timor2023implicit}
Nadav Timor, Gal Vardi, and Ohad Shamir.
\newblock Implicit regularization towards rank minimization in relu networks.
\newblock In \emph{International Conference on Algorithmic Learning Theory},
  pages 1429--1459. PMLR, 2023.

\bibitem[Tsigler and Bartlett(2020)]{tsigler2020benign}
Alexander Tsigler and Peter~L Bartlett.
\newblock Benign overfitting in ridge regression.
\newblock \emph{arXiv preprint arXiv:2009.14286}, 2020.

\bibitem[Vardi(2022)]{vardi2022implicit}
Gal Vardi.
\newblock On the implicit bias in deep-learning algorithms.
\newblock \emph{arXiv preprint arXiv:2208.12591}, 2022.

\bibitem[Vaskevicius et~al.(2019)Vaskevicius, Kanade, and
  Rebeschini]{vaskevicius2019implicit}
Tomas Vaskevicius, Varun Kanade, and Patrick Rebeschini.
\newblock Implicit regularization for optimal sparse recovery.
\newblock \emph{Advances in Neural Information Processing Systems}, 32, 2019.

\bibitem[Vershynin(2018)]{vershynin2018high}
Roman Vershynin.
\newblock \emph{High-dimensional probability: An introduction with applications
  in data science}, volume~47.
\newblock Cambridge university press, 2018.

\bibitem[Wainwright(2019)]{wainwright2019high}
Martin~J Wainwright.
\newblock \emph{High-dimensional statistics: A non-asymptotic viewpoint},
  volume~48.
\newblock Cambridge University Press, 2019.

\bibitem[Wang et~al.(2022)Wang, Donhauser, and Yang]{wang2022tight}
Guillaume Wang, Konstantin Donhauser, and Fanny Yang.
\newblock Tight bounds for minimum {$\ell_1$}-norm interpolation of noisy data.
\newblock In \emph{International Conference on Artificial Intelligence and
  Statistics}, pages 10572--10602. PMLR, 2022.

\bibitem[Woodworth et~al.(2020)Woodworth, Gunasekar, Lee, Moroshko, Savarese,
  Golan, Soudry, and Srebro]{woodworth2020kernel}
Blake Woodworth, Suriya Gunasekar, Jason~D Lee, Edward Moroshko, Pedro
  Savarese, Itay Golan, Daniel Soudry, and Nathan Srebro.
\newblock Kernel and rich regimes in overparametrized models.
\newblock In \emph{Conference on Learning Theory}, pages 3635--3673. PMLR,
  2020.

\bibitem[Yun et~al.(2021)Yun, Krishnan, and Mobahi]{yun2021a}
Chulhee Yun, Shankar Krishnan, and Hossein Mobahi.
\newblock A unifying view on implicit bias in training linear neural networks.
\newblock In \emph{International Conference on Learning Representations}, 2021.
\newblock URL \url{https://openreview.net/forum?id=ZsZM-4iMQkH}.

\bibitem[Zhou et~al.(2020)Zhou, Sutherland, and Srebro]{zhou2020uniform}
Lijia Zhou, Danica~J Sutherland, and Nati Srebro.
\newblock On uniform convergence and low-norm interpolation learning.
\newblock \emph{Advances in Neural Information Processing Systems},
  33:\penalty0 6867--6877, 2020.

\end{thebibliography}
\newpage
\onecolumn
\appendix
\section{Preliminary}
In this section, we prepare some useful lemmas for the later analysis. In Section~\ref{sec: appendix rip}, we show that Assumption~\ref{assump: 1} is true when data matrix $\mX$ is a Gaussian random matrix and noise $\vxi\sim N(0,\sigma^2\mI)$. In Section~\ref{sec: appendix grad approx pf}, we give the proof of Lemma~\ref{lem: rip approx} and Lemma~\ref{lem: grad approx} for gradient approximation.

\subsection{RIP and regularity conditions}\label{sec: appendix rip}
In this subsection, we show that Assumption~\ref{assump: 1} can be satisfied when $\mX$ is a Gaussian random matrix and $\vxi$ is a Gaussian random vector with variance $\sigma^2$.

We use standard proof to show the RIP property, and the rest of the properties follow from simple concentration.
First, the following shows random Gaussian matrix is a $(s+1,\delta)$-RIP matrix with $\delta=\Theta(\sqrt{(s/n)\log(d/s)})$. To satisfy Assumption~\ref{assump: 1}, with simple calculation we see that we only require $\tldOmega(s^4)\le n\le\tldO(d/s^4)$.
\begin{proposition}\label{prop: rip}
    Let $\mX$ be a $n\times d$ Gaussian random matrix. Then, there exists universal constant $c_1,c_2$ such that $\mX/\sqrt{n}$ is $(k,\delta)$-RIP for any $k\le c_1n/\log(d/k)$ and $\delta\ge c_2\sqrt{(k/n)\log(d/k)}$ with probability at least $1-(k/d)^k\ge 1-1/d$.
\end{proposition}

\begin{proof}
    From the proof of Theorem 5.2 in \citep{baraniuk2008simple}, we know the error probability is at most
    \[
        e^{-c_0(\delta/2)n+k[\log(e d/k)+\log(12/\delta)]+\log(2)},    
    \]
    where $c_0(\eps)=\eps^2/4-\eps^3/6$. Note that it suffices to consider $\delta<1$, which implies that $c_0(\delta/2)\ge\delta^2/48$ and $k\le n/c_2^2/\log(d/k)$. Then the exponent can be upper bounded by with $\delta\ge c_2\sqrt{(k/n)\log(d/k)}$
    \[
        -n\delta^2/48 + (4+\log(1/c_2)) k\log(d/k) \le -(c_2^2/48) k\log(d/k)+(4+\log(1/c_2))k\log(d/k) <-(c_2^2/50) k\log(d/k),
    \]
    where the last inequality is true since we can choose $c_2$ to be large enough constant.
\end{proof}

The following lemma shows that the regularity conditions on $\mX,\vxi$ in the second part of Assumption~\ref{assump: 1} are satisfied with high probability when $\mX$ is a Gaussian matrix and $\vxi$ is sampled from $N(0,\sigma^2\mI)$.
\begin{lemma}[Regularity conditions]\label{lem: reg condition}
Suppose $\mX$ is a Gaussian matrix and $\vxi\sim N(0,\sigma^2\mI)$. With probability at least $1-de^{-\Omega(n)}$, We have 
\begin{align*}
    &\norm{\vxi}_2=O(\sigma\sqrt{n}),\\
    &\norm{\frac{1}{n}\mX^\top\vxi}_\infty\le B_\xi:=O\left(\sigma\sqrt{\frac{\log d}{n}}\right),\\
    &\norm{\mX^\top\vxi}_2=O\left(\sigma\sqrt{dn}\right),\\
    &\norm{\frac{1}{n}\mX^\top\vbeta}_\infty=O\left(\frac{\norm{\vbeta}_2}{\sqrt{n}}\right)\text{ for any vector $\vbeta$},\\
    &(1-O(\sqrt{n/d}))d\le\lambda_{\rm{min}}(\mX \mX^\top)\le\lambda_{\rm{max}} (\mX \mX^\top)\le (1+O(\sqrt{n/d}))d.
\end{align*}
\end{lemma}
\begin{proof}
    The first three and the last one are standard consequences of Gaussian vector/matrix concentration, see e.g., Lemma A.5 in \citet{vaskevicius2019implicit} for the proof of $\norm{\mX^\top\vxi/n}_\infty$ and Theorem 3.1.1 and Theorem 4.4.5 in \citet{vershynin2018high} for the rest. For the third one, denote $\mX[:,i]$ is the $i$-th column of $\mX$. Then, $\norm{\mX^\top\vbeta/n}_\infty\le \max_i |\vbeta^\top\mX[:,i]|/n\le \norm{\vbeta}_2 \max_i \norm{\mX[:,i]}_2/n$. Then it follows from standard Gaussian concentration.
\end{proof}

Now we are ready to prove Lemma~\ref{lem: rip gaussian} that shows Assumption~\ref{assump: 1} holds under Gaussian input and 
Gaussian noise. It immediately follows from Proposition~\ref{prop: rip} and Lemma~\ref{lem: reg condition} above.
\lemripgaussian*
\begin{proof}
    It suffices to combine Proposition~\ref{prop: rip} and Lemma~\ref{lem: reg condition}.
\end{proof}

\subsection{Gradient approximation}\label{sec: appendix grad approx pf}
 Lemma~\ref{lem: rip approx} and Lemma~\ref{lem: grad approx} give ways to approximate several important terms in the gradient. Here we give their proofs.

\lemripapprox*
\begin{proof}
For the first part, it is a standard consequence of the RIP condition, see e.g., Lemma A.3 in \citet{vaskevicius2019implicit}. For the second part, notice that $\vbeta=\sum_i \beta_i \ve_i$ where $\{\ve_i\}_{i=1}^d$ is the standard basis, it then follows from the first part.
\end{proof}

\lemgradapprox*

\begin{proof}
    By the decomposition of $\mX^\top\mX\vv/n$ in \eqref{eq: decomp xxv}, we have
    \begin{align*}
        \frac{1}{n}\mX^\top\vr
        =& \frac{1}{n}\mX^\top\mX\vv-\frac{1}{n}\mX^\top\vxi
        + \frac{1}{n}\mX^\top\mX(\lambda \vw_{S_+}^{\odot 2} - \lambda \vu_{S_-}^{\odot 2}-\vbeta^*) 
        + \frac{1}{n}\mX^\top(\lambda \mX\vw_{e_+}^{\odot2}-\lambda \mX\vu_{e_-}^{\odot2})\\
        =& (b-\frac{1}{n}) (\mX^\top\vxi)_e  -\frac{1}{n}(\mX^\top\vxi)_S
        + \frac{d}{n} \vv_S+ \frac{1}{n}\mX^\top\mX(\lambda \vw_{S_+}^{\odot 2} - \lambda \vu_{S_-}^{\odot 2}-\vbeta^*)\\ 
        &+ \frac{1}{n}\mX^\top(\lambda \mX\vw_{e_+}^{\odot2}-\lambda \mX\vu_{e_-}^{\odot2}) + \vGamma\\
        =& (b-\frac{1}{n}) (\mX^\top\vxi)_e  -\frac{1}{n}(\mX^\top\vxi)_S
        + \frac{d}{n} \vv_S+ \lambda \vw_{S_+}^{\odot 2} - \lambda \vu_{S_-}^{\odot 2}-\vbeta^*\\ 
        &+(\frac{1}{n}\mX^\top\mX-\mI)(\frac{d}{n}\vv_S+\lambda \vw_{S_+}^{\odot 2} - \lambda \vu_{S_-}^{\odot 2}-\vbeta^*)
        -(\frac{1}{n}\mX^\top\mX-\mI)\frac{d}{n}\vv_S\\ 
        &+ \frac{1}{n}\mX^\top(\lambda \mX\vw_{e_+}^{\odot2}-\lambda \mX\vu_{e_-}^{\odot2}) + \vGamma\\
        =& \frac{d}{n} \vv_S+ \lambda \vw_{S_+}^{\odot 2} - \lambda \vu_{S_-}^{\odot 2}-\vbeta^*\\
        &+ (b-\frac{1}{n}) (\mX^\top\vxi)_e  -\frac{1}{n}(\mX^\top\vxi)_S
        +(\frac{1}{n}\mX^\top\mX-\mI)(\frac{d}{n}\vv_S+\lambda \vw_{S_+}^{\odot 2} - \lambda \vu_{S_-}^{\odot 2}-\vbeta^*)
        -(\frac{1}{n}\mX^\top\mX-\mI)\frac{d}{n}\vv_S\\ 
        &+ \frac{1}{n}\mX^\top(\lambda \mX\vw_{e_+}^{\odot2}-\lambda \mX\vu_{e_-}^{\odot2}) + \vGamma.
    \end{align*}
    Denote the last two lines in in the last equation above as $\vDelta_r$. We know
    \begin{align*}
        \frac{1}{n}\mX^\top\vr
        = \frac{d}{n} \vv_S+ \lambda \vw_{S_+}^{\odot 2} - \lambda \vu_{S_-}^{\odot 2}-\vbeta^* + \vDelta_r.
    \end{align*}
    To bound $\norm{\vDelta_r}_\infty$, by Lemma~\ref{lem: rip approx} and Assumption~\ref{assump: 1}, we know
    \begin{align*}
        \norm{(b-\frac{1}{n}) (\mX^\top\vxi)_e  -\frac{1}{n}(\mX^\top\vxi)_S}_\infty
        =&O\left(\left(1+\left|nb-1\right|\right)B_\xi\right)\\
        \norm{(\frac{1}{n}\mX^\top\mX-\mI)(\frac{d}{n}\vv_S+\lambda \vw_{S_+}^{\odot 2} - \lambda \vu_{S_-}^{\odot 2}-\vbeta^*)}_\infty
        =&\sqrt{s}\delta\norm{\frac{d}{n}\vv_S+\lambda \vw_{S_+}^{\odot 2} - \lambda \vu_{S_-}^{\odot 2}-\vbeta^*}_2\\
        \le& s\delta\norm{\frac{d}{n}\vv_S+\lambda \vw_{S_+}^{\odot 2} - \lambda \vu_{S_-}^{\odot 2}-\vbeta^*}_\infty\\
        \norm{(\frac{1}{n}\mX^\top\mX-\mI)\frac{d}{n}\vv_S}_\infty
        \le&\sqrt{s}\delta \frac{d}{n}\norm{\vv_S}_2\\
        \norm{\frac{1}{n}\mX^\top(\lambda \mX\vw_{e_+}^{\odot2}-\lambda \mX\vu_{e_-}^{\odot2})}_\infty
        =&O\left(\frac{\lambda}{\sqrt{n}}\norm{\mX\vw_{e_+}^{\odot2}-\mX\vu_{e_-}^{\odot2}}_2\right)\\
        =&O\left(\frac{d}{\sqrt{n}}\lambda\right)(\norm{\vw_{e_+}}_\infty^2+\norm{\vu_{e_-}}_\infty^2),
    \end{align*}
    Thus, we have
    \begin{align*}
        \norm{\vDelta_r}_\infty
        =& O\left(\left(1+\left|nb-1\right|\right)B_\xi\right)
        +s\delta\norm{\frac{d}{n}\vv_S+\lambda \vw_{S_+}^{\odot 2} - \lambda \vu_{S_-}^{\odot 2}-\vbeta^*}_\infty
        +\sqrt{s}\delta \frac{d}{n}\norm{\vv_S}_2\\
        &+O\left(\frac{d}{\sqrt{n}}\lambda\right)(\norm{\vw_{e_+}}_\infty^2+\norm{\vu_{e_-}}_\infty^2)
        +\gamma
    \end{align*}
\end{proof}
\section{Proof for Stage 1}\label{sec: pf stage 1}

Recall that our goal in Stage 1 is to show (1) variables $\vw_S$ and $\vu_S$ grow large to recover $\vbeta^*$ (specifically, $\vw_S$ recovers the positive entries of $\vbeta^*$ and $\vu_S$ recovers the negative entries of $\vbeta^*$); (2) the other variables $\vw_e$, $\vu_e$ and $\vv$ remain small. This is summarized in the following main lemma:

\lemstageonemain*

To prove this lemma, we need to maintain the following inductive hypothesis which assumes the approximate error comes from the non-signal entry is small and other regularity conditions. Later we will use these assumptions to bound different error terms and finish the induction.
\begin{restatable}[Inductive Hypothesis for Stage 1]{lemma}{lemstageoneIH}\label{lem: IH stage 1}
    For $t\le \tldT_1:=C_{\tldT_1}\log(1/\alpha)/\eta\lambda\betamin$ with large enough universal constant $C_{\tldT_1}$, the following hold with large enough universal constant $\tldC_1$:
    \begin{itemize}
        \item $\norm{\vw_{e_+}^\supt{t}}_\infty, \norm{\vu_{e_-}^\supt{t}}_\infty \le \tldC_1\alpha$.
        \item $\norm{\lambda \vw_{S_+}^{\supt{t}\odot2}-\lambda\vu_{S_-}^{\supt{t}\odot2}-\vbeta^*}_\infty\le \tldC_1$, $\norm{\frac{d}{n}\vv_S^\supt{t}+\lambda \vw_{S_+}^{\supt{t}\odot2}-\lambda\vu_{S_-}^{\supt{t}\odot2}-\vbeta^*}_\infty\le \tldC_1$.
        \item $\norm{\vr^\supt{t}}_2\le\norm{\vr^\supt{0}}_2\le \tldC_1\sqrt{sn}$. 
    \end{itemize}
\end{restatable}
We finally remark on the constant dependency in Stage 1 here. All the lemmas appear in this section, except the main result Lemma~\ref{lem: stage 1}, should only depend on universal constants $C_{T_1},\tldC_1,C_{\tldT_1}$. Perhaps the one that needs the most careful proof is the induction hypothesis Lemma~\ref{lem: IH stage 1}. In the proof, we rely on the condition $\tldOmega(s)\le n\le \tldO(d/s)$ to make sure that the terms like $B_\xi+\sigma\sqrt{n/d}$ are smaller than any universal constant, especially the constant that only depends on universal constants $C_{T_1},\tldC_1,C_{\tldT_1}$. In this way, it ensures that we can choose another universal constant $C_1$ that only depends on $C_{T_1},\tldC_1,C_{\tldT_1}$ to serve as the upper bound in Lemma~\ref{lem: stage 1}.

\subsection{Dynamics of $v$}\label{subsec: stage 1 v dynamic}

As we discussed earlier, even though in Stage 1 we hope to use the corresponding entries of $\vu$, $\vw$ to learn the signal, the same entries of $\vv$ will also grow and it's important to understand the dynamics of $\vv$. 

The dynamics of $\vv$ roughly follows the trajectory for optimizing $\norm{\mX\vv-\vxi}_2^2/2n$. We formalize that in the following two lemmas. First, we give a decomposition of $\mX\mX^\top\vv/n$ as follow. This term plays an important role when we estimate the gradient as shown in Lemma~\ref{lem: grad approx}, therefore we here give a careful analysis.
\begin{restatable}{lemma}{lemstageonexxvdynamic}\label{lem: stage 1 xxv dynamic}
    Recall the decomposition in \eqref{eq: decomp xxv}
    \begin{align*}
        \frac{1}{n}\mX^\top\mX\vv^\supt{t}&= \frac{d}{n} \vv_S^\supt{t} + b_t (\mX^\top\vxi)_e + \vGamma_t,\\
        b_{t+1}&=b_t-\frac{\eta d}{n}\left(b_t-\frac{1}{n}\right),
    \end{align*}
    where $\norm{\vGamma^\supt{t}}_\infty\le \gamma_t$ and recall the notation $\vbeta_S=\sum_{i:\beta^*_i\ne 0}\beta_i\ve_i$, $\vbeta_e=\sum_{i:\beta^*_i= 0}\beta_i\ve_i$.
    Suppose Lemma~\ref{lem: IH stage 1} holds. We have for $t\le \tldT_1$
    \begin{align*}
        b_t &= (1-(1-\eta d/n)^t)/n\le 1/n,\\
        \gamma_t &\le O((\sqrt{sd/n}+ds\delta/n)\eta t) = O(\sigma\sqrt{n/d}+B_\xi).
    \end{align*}
\end{restatable}

We then give the decomposition of $\vv$.
\begin{restatable}{lemma}{lemstageonevdynamic}\label{lem: stage 1 v dynamic}
    Recall the decomposition in \eqref{eq: decomp v}
    \begin{align*}
        \vv^\supt{t}&= \vv_S^\supt{t} + a_t \mX^\top\vxi + \vDelta_v^\supt{t},\\
        a_{t+1}&=a_t - \eta (b_t - 1/n)
    \end{align*}
    where $\norm{\vDelta_v^\supt{t}}_\infty\le \zeta_t$. and recall the notation $\vbeta_S=\sum_{i:\beta^*_i\ne 0}\beta_i\ve_i$, $\vbeta_e=\sum_{i:\beta^*_i= 0}\beta_i\ve_i$. 
    Suppose Lemma~\ref{lem: IH stage 1} holds. We have for $t\le \tldT_1$
    \begin{align*}
        a_t&=(1-(1-\eta d/n)^t)/d\le 1/d\\
        \zeta_t&=O((B_\xi+s\delta+\sigma\sqrt{n/d})\eta t)=O(\sigma\sqrt{n}/d).
    \end{align*}
    Moreover, for every $t\le \tldT_1$, $\norm{\vv^\supt{t}}_2=O(\sigma\sqrt{n/d})$, $\norm{\vv_S^\supt{t}}_2=O(\sqrt{s}(n/d)\log^2(d)(B_\xi+\sigma\sqrt{n/d}))$.
\end{restatable}

\subsection{Implications of Inductive Hypothesis Lemma~\ref{lem: IH stage 1}}\label{subsec: IH stage 1 implications}

Given the understanding of dynamics of $\vv$ and $\mX^\top\mX \vv$ in Appendix~\ref{subsec: stage 1 v dynamic}, we have the following approximation of gradient, using Lemma~\ref{lem: grad approx}.
\begin{restatable}{lemma}{lemstageonegradapprox}\label{lem: stage 1 grad approx}
    In the setting of Lemma~\ref{lem: stage 1 xxv dynamic} and Lemma~\ref{lem: stage 1 v dynamic}, we have for $t\le \tldT_1$
    \begin{align*}
        \nabla_\vw L &= \left(\frac{1}{n} \mX^\top \vr \right)\odot(2\lambda \vw)
        =2\lambda (\frac{d}{n} \vv_S + \lambda \vw^{\odot 2}_{S_+} - \lambda \vu^{\odot2}_{S_-}-\vbeta^*+\vDelta_r)\odot \vw,\\
        \nabla_\vu L &= -\left(\frac{1}{n} \mX^\top \vr \right)\odot(2\lambda \vu)
        =-2\lambda (\frac{d}{n} \vv_S + \lambda \vw^{\odot 2}_{S_+} - \lambda \vu^{\odot 2}_{S_-}-\vbeta^*+\vDelta_r)\odot \vu,\\
        \nabla_v L &= \frac{1}{n} \mX^\top \vr
        =\frac{d}{n} \vv_S + \lambda \vw^{\odot 2}_{S_+} - \lambda \vu^{\odot 2}_{S_-}-\vbeta^*+\vDelta_r,
    \end{align*}
    where 
    \begin{align*}
        \norm{\vDelta_r^\supt{t}}_\infty
        =& \underbrace{O\left(B_\xi+\sigma\sqrt{n/d}\right)}_{=:B_s}
        +s\delta\norm{\frac{d}{n}\vv_S^\supt{t}+\lambda \vw_{S_+}^{\supt{t}\odot 2} - \lambda \vu_{S_-}^{\supt{t}\odot 2}-\vbeta^*}_\infty,
    \end{align*} 
\end{restatable}

Now we are ready to estimate the dynamics for the relevant entries of $\vu$ and $\vw$ using Lemma~\ref{lem: stage 1 grad approx}. We first show in Lemma~\ref{lem: stage 1 first grow} that $\vw_{S_+},\vu_{S_-}$ will grow to $\Omega(\betamin)$. Then in Lemma~\ref{lem: stage 1 decrease half} we show that it takes $O(1/\eta\lambda\betamin)$ to decrease $\norm{\frac{d}{n}\vv_S^\supt{t}+\lambda (\vw_{S_+}^\supt{t})^2-\lambda (\vu_{S_-}^\supt{t})^2 - \vbeta^*}_\infty$ by half. The proofs are deferred to Appendix~\ref{subsec: appendix stage 1 pf}.

\begin{restatable}{lemma}{lemstageonefirstgrow}\label{lem: stage 1 first grow}
    Suppose Lemma~\ref{lem: IH stage 1} hold. Then for every $T_{11}\le t\le\tldT_1$ with $T_{11}=O(\log(1/\lambda\alpha^2)/\eta\lambda\betamin)$, $\lambda(w_k^\supt{t})^2\ge\betamin/4$ for $k\in S_+$ and $\lambda(u_k^\supt{t})^2\ge\betamin/4$ for $k\in S_-$.
\end{restatable}

\begin{restatable}{lemma}{lemstageonedecreasehalf}\label{lem: stage 1 decrease half}
    Suppose Lemma~\ref{lem: IH stage 1} and Lemma~\ref{lem: stage 1 first grow} hold. Given any time $t_0$, assume at time $t_0$ $B_0:=\norm{\frac{d}{n}\vv_S^\supt{t_0}+\lambda (\vw_{S_+}^\supt{t_0})^2-\lambda (\vu_{S_-}^\supt{t_0})^2 - \vbeta^*}_\infty\ge 4B_s$ where $B_s$ is defined in Lemma~\ref{lem: stage 1 grad approx}. Let 
    \[
        T':=\inf\left\{t-t_0\ge 0:\norm{\frac{d}{n}\vv_S^\supt{t}+\lambda (\vw_{S_+}^\supt{t})^2-\lambda (\vu_{S_-}^\supt{t})^2 - \vbeta^*}_\infty \le \norm{\frac{d}{n}\vv_S^\supt{t_0}+\lambda (\vw_{S_+}^\supt{t_0})^2-\lambda (\vu_{S_-}^\supt{t_0})^2 - \vbeta^*}_\infty/2\right\}
    \]
    be the time that signal error reduces by half.
    Then, we know $T'=O(1/\eta\lambda\betamin)$.
\end{restatable}    

As a technical condition in proving the two lemmas above, we need to make sure that once we fit the signal using the corresponding entries in $\vu,\vw,\vv$ up to error $\mu$, the error will not become much worse later. We formalize this as the following stability lemma.
\begin{restatable}[Stability]{lemma}{lemstability}\label{lem: stability}
    Suppose Lemma~\ref{lem: stage 1 grad approx} and Lemma~\ref{lem: stage 1 first grow} hold. Assume $\norm{\frac{d}{n}\vv^\supt{t_0} + \lambda \vw^{\supt{t_0}\odot2} -\lambda \vu^{\supt{t_0}\odot2} - \vbeta^*}_\infty\le \mu$ at time $t_0$, then 
    $|\frac{d}{n}v_k^\supt{t} + \lambda (w_k^\supt{t})^2 -\lambda (u_k^\supt{t})^2 - \beta_k^*|\le \max\{\mu, 2(B_s+s\delta\mu)\}$ for all $t\ge t_0$ and $k\in S$, where $B_s$ is definied in Lemma~\ref{lem: stage 1 grad approx}.
\end{restatable}

Now we are ready to bound the time $T_1$ for Stage 1 using the above lemmas.
\begin{restatable}{lemma}{lemstageonetime}\label{lem: stage 1 time}
    Suppose Lemma~\ref{lem: IH stage 1} hold. Recall 
    \[
        T_1:=\inf\left\{t:\norm{\frac{d}{n}\vv_S^\supt{t}+\lambda (\vw_{S_+}^\supt{t})^2-\lambda (\vu_{S_-}^\supt{t})^2 - \vbeta^*}_\infty \le C_{T_1}(B_\xi+\sigma\sqrt{n/d})\right\},
    \]
    where $C_{T_1}$ is a large enough universal constant.
    Then, we know $T_1=O(\log(1/\alpha)/\eta\lambda\betamin)$.
\end{restatable}

\subsection{Proof of Inductive Hypothesis Lemma~\ref{lem: IH stage 1} and Lemma~\ref{lem: stage 1}}
Finally, we are ready to prove in the induction hypothesis and finish the proof of Lemma~\ref{lem: stage 1}.

\lemstageoneIH*
\begin{proof}
    We claim  $\norm{\vw_{e_+}^\supt{t}}_\infty, \norm{\vu_{e_-}^\supt{t}}_\infty = \alpha(1+O(\eta\lambda(B_\xi+\sigma\sqrt{n/d}+s\delta)))^t$ and $\norm{\frac{d}{n}\vv_S^\supt{t}+\lambda \vw_{S_+}^{\supt{t}\odot 2}-\lambda \vu_{S_-}^{\supt{t}\odot 2}-\vbeta^*}_\infty\le \tldC_1/2$. If such claim is true, then we prove the first 2 points as $t\le \tldT_1=o(\eta^{-1}\lambda^{-1}(B_\xi+\sigma\sqrt{n/d}+s\delta)^{-1})$ and $(d/n)\norm{\vv_S}_\infty\le \tldC_1/2$. The latter one is true because we can choose $\tldC_1$ to be large and use Lemma~\ref{lem: stage 1 v dynamic}. Also, the condition $\tldOmega(s)\le n\le \tldO(d/s)$ makes sure that we can ensure $B_\xi+\sigma\sqrt{n/d}$ to be smaller than any universal constant, especially the constant that only depends on universal constants $C_{T_1},\tldC_1,C_{\tldT_1}$ in this case. 
    
    We show the above claim by induction. We know all the conditions hold at $t=0$. Suppose before time $t$ it holds, then consider the time $t+1$.
    
    For $\norm{\frac{d}{n}\vv_S^\supt{t+1}+\lambda \vw_{S_+}^{\supt{t+1}\odot 2}-\lambda \vu_{S_-}^{\supt{t+1}\odot 2}-\vbeta^*}_\infty$, if $\lambda (w_k^\supt{t+1})^2+\lambda (u_k^\supt{t+1})^2\le\betamin/4$, then it is easy to see it is bounded by $\tldC_1/2$. Otherwise, we can see it from the proof of Lemma~\ref{lem: stage 1 decrease half} and Lemma~\ref{lem: stage 1 time} that shows it continues to decrease and stay small. 
    
    Now consider $\norm{\vw_{e_+}^\supt{t+1}}_\infty$ and $\norm{\vu_{e_-}^\supt{t+1}}_\infty$. By Lemma~\ref{lem: stage 1 grad approx} we have for $k\not\in S$ 
    \begin{align*}
        |w_k^\supt{t+1}| 
        \le& 
        (1+ 2\lambda\eta O(B_\xi+\sigma\sqrt{n/d}+s\delta))|w_k^\supt{t}|,
    \end{align*}
    which implies that $|w_k^\supt{t+1}|\le (1+  O(\lambda\eta(B_\xi+\sigma\sqrt{n/d}+s\delta)))^{t+1}\alpha$ as $w_k^\supt{0}=\alpha$. Similarly, we get the same bound for $u_k$ with $k\not\in S$.

    It remains to consider $w_k$ with $k\in S_-$ and $u_k$ with $k\in S_+$. We will focus on $w_k$ with $k\in S_-$, the other follows the same calculation. We have
    \begin{align*}
        w_k^\supt{t+1} u_k^\supt{t+1}
        =& \left(1-2\eta\lambda\left(\frac{1}{n}\mX^\top \vr^\supt{t}\right)_k\right)w_k^\supt{t} 
        \cdot \left(1+2\eta\lambda\left(\frac{1}{n}\mX^\top \vr^\supt{t}\right)_k\right)u_k^\supt{t}
        \le w_k^\supt{t} u_k^\supt{t}\le \alpha^2.
    \end{align*}
    From the proof of Lemma~\ref{lem: stage 1 time} we know $u_k^\supt{t}\ge\alpha$. This implies that $|w_k^\supt{t}|\le \alpha$. 

    We now prove the last part on $\norm{\vr^\supt{t+1}}_2$. We have
    \begin{align*}
        \vr^\supt{t+1} 
        =& \mX\vv^\supt{t+1} + \lambda \mX \vw^{\supt{t+1}\odot2}-\lambda \mX \vu^{\supt{t+1}\odot2} - \vxi\\
        =& \vr^\supt{t} - \eta \mX \cdot \frac{1}{n} \mX^\top \vr^\supt{t} 
        + \lambda \mX\left(-\eta\frac{4\lambda}{n}(\mX^\top \vr)\odot \vw^{\odot 2}
        + \eta^2 \frac{4\lambda^2}{n^2} (\mX^\top \vr)^{\odot 2}\odot \vw^{\odot 2}
        \right)\\
        &- \lambda \mX\left(\eta\frac{4\lambda}{n}(\mX^\top \vr)\odot \vu^{\odot 2}
        + \eta^2 \frac{4\lambda^2}{n^2} (\mX^\top \vr)^{\odot 2}\odot \vu^{\odot 2}
        \right).
    \end{align*}
    This suggests that
    \begin{align*}
        \norm{\vr^\supt{t+1}}_2
        \le& \left(1-\frac{\eta}{n}\lambda_{\rm{min}}(\mX\mX^\top)-\frac{4\eta\lambda^2}{n}\lambda_{\rm{min}}(\mX\diag(\vw^{\odot2}+\vu^{\odot2}) \mX^\top)\right)\norm{\vr^\supt{t}}_2
        +\lambda \sqrt{d}\cdot O\left(
        \eta^2 \frac{\lambda^2}{n^2} d\norm{\vr}_2^2
        \right)\\
        \le& \left(1-\Omega\left(\frac{\eta d}{n}\right)\right)\norm{\vr^\supt{t}}_2
    \end{align*}
    where we use Lemma~\ref{lem: odot bound} and Assumption~\ref{assump: 1}. We finish the induction.

\end{proof}

Now we are ready to proof the main result Lemma~\ref{lem: stage 1} for Stage 1.
\lemstageonemain*

\begin{proof}
    Combine Lemma~\ref{lem: IH stage 1}, Lemma~\ref{lem: stage 1 v dynamic}, Lemma~\ref{lem: stage 1 time} we prove the first 2 points and bound the time $T_1$. For the last point, by Lemma~\ref{lem: rip approx} and Assumption~\ref{assump: 1}
    \begin{align*}
        \norm{\vr^\supt{T_1}}_2
        \le& 
        \norm{\mX\lambda \vw_{e_+}^{\supt{T_1}\odot 2}}_2
        +\norm{\mX\lambda \vu_{e_-}^{\supt{T_1}\odot 2}}_2
        +\norm{\mX(\vv_S^\supt{T_1}+\lambda \vw_{S_+}^{\supt{T_1}\odot 2}-\lambda \vu_{S_-}^{\supt{T_1}\odot 2}-\vbeta^*)}_2
        +\norm{\mX(\vv^\supt{T_1}-\vv_S^\supt{T_1})-\vxi}_2\\
        \le& O(\lambda d\alpha^2)+O(\sqrt{ns}(B_\xi+\sigma\sqrt{n/d}))
        +(d/n-1)\norm{\mX\vv_S^\supt{T_1}}_2+\norm{(a_{T_1} \mX\mX^\top-\mI)\vxi+\vDelta_v^\supt{T_1}}_2\\
        \le& O(1)\norm{\vxi}_2 + O(\lambda d\alpha^2)+O(\sqrt{ns}(B_\xi+\sigma\sqrt{n/d}))
        +\tldO(\sqrt{ns}(B_\xi+\sigma \sqrt{n/d}))+\sqrt{d}\zeta_{T_1}\\
        =& O(\sigma\sqrt{n}),
    \end{align*}
    where we use $a_{T_1}\le 1/d$ and $\zeta_{T_1}=O(\sigma\sqrt{n}/d)$ from Lemma~\ref{lem: stage 1 v dynamic}. Note that the constants hide in big-$O$ in these lemmas only depends on universal constants $C_{T_1},\tldC_1,C_{\tldT_1}$, so we can choose another large enough universal constant $C_1$ to appear at these upper bounds.
\end{proof}

\subsection{Omitted Proofs in Section~\ref{subsec: stage 1 v dynamic} and Section~\ref{subsec: IH stage 1 implications}}\label{subsec: appendix stage 1 pf}
In this subsection, we give the proof of Lemma~\ref{lem: stage 1 xxv dynamic}, Lemma~\ref{lem: stage 1 v dynamic}, Lemma~\ref{lem: stage 1 grad approx}, Lemma~\ref{lem: stage 1 first grow}, Lemma~\ref{lem: stage 1 decrease half}, Lemma~\ref{lem: stability} and Lemma~\ref{lem: stage 1 time} in previous subsections.

\lemstageonexxvdynamic*
\begin{proof}
    We first write the update of $b_t$ and $\vGamma_t$ using the update of $\vv$.
    \begin{align*}
        b_{t+1} (\mX^\top\vxi)_e + \vGamma_{t+1} 
        =& \frac{1}{n}\mX^\top\mX\vv^\supt{t+1} -  \frac{d}{n} \vv_S^\supt{t+1} \\
        =& \frac{1}{n}\mX^\top\mX\vv^\supt{t} - \frac{d}{n} \vv_S^\supt{t} 
        - \eta \frac{1}{n}\mX^\top\mX \frac{1}{n}\mX^\top \vr^\supt{t} 
        + \eta\frac{d}{n} \left(\frac{1}{n}\mX^\top\vr^\supt{t}\right)_S\\
        =& b_{t} (\mX^\top\vxi)_e + \vGamma_{t} - \frac{\eta}{n^2}\mX^\top\mX \mX^\top \vr^\supt{t} 
        + \eta\frac{d}{n} \left(\frac{1}{n}\mX^\top\vr^\supt{t}\right)_S\\
        =& b_{t} (\mX^\top\vxi)_e + \vGamma_{t} - \frac{\eta}{n^2}\mX^\top(\mX \mX^\top -d\mI)\vr^\supt{t} 
        - \eta\frac{d}{n} \left(\frac{1}{n}\mX^\top\vr^\supt{t}\right)_e.
    \end{align*}

    We bound the last two terms one by one. For $\frac{\eta}{n^2}\mX^\top(\mX \mX^\top -d\mI)\vr^\supt{t}$, we have by Assumption~\ref{assump: 1} and Lemma~\ref{lem: IH stage 1}
    \begin{align*}
        \norm{\frac{\eta}{n^2}\mX^\top(\mX \mX^\top -d\mI)\vr^\supt{t}}_\infty
        \le \frac{\eta}{n} O(\frac{1}{\sqrt{n}}\cdot \sqrt{dn}\cdot \sqrt{sn})
        =O(\eta\sqrt{sd/n}).
    \end{align*}
    For $\eta\frac{d}{n} \left(\frac{1}{n}\mX^\top\vr^\supt{t}\right)_e$, we have
    \begin{align*}
        &\left(\frac{1}{n}\mX^\top\vr^\supt{t}\right)_e\\
        =&\left(\frac{1}{n}\mX^\top\mX\vv^\supt{t} 
        + \frac{1}{n}\mX^\top\mX(\lambda \vw_{S_+}^{\supt{t}\odot 2} - \lambda \vu_{S_-}^{\supt{t}\odot 2}-\vbeta^*) - \frac{1}{n}\mX^\top\vxi+\frac{1}{n}\mX^\top(\lambda \mX\vw_{e_+}^{\supt{t}\odot2}-\lambda \mX\vu_{e_-}^{\supt{t}\odot2})\right)_e\\
        =&\left(\frac{d}{n}\vv_S^\supt{t} 
        + \frac{1}{n}\mX^\top\mX(\lambda \vw_{S_+}^{\supt{t}\odot 2} - \lambda \vu_{S_-}^{\supt{t}\odot 2}-\vbeta^*) +(b_t- \frac{1}{n})\mX^\top\vxi+ \vGamma_t+\frac{1}{n}\mX^\top(\lambda \mX\vw_{e_+}^{\supt{t}\odot2}-\lambda \mX\vu_{e_-}^{\supt{t}\odot2})\right)_e\\
        =&(b_t- \frac{1}{n})(\mX^\top\vxi)_e+\left(
        (\frac{1}{n}\mX^\top\mX-\mI)(\lambda \vw_{S_+}^{\supt{t}\odot 2} - \lambda \vu_{S_-}^{\supt{t}\odot 2}-\vbeta^*) + \vGamma_t+\frac{1}{n}\mX^\top(\lambda \mX\vw_{e_+}^{\supt{t}\odot2}-\lambda \mX\vu_{e_-}^{\supt{t}\odot2})\right)_e.
    \end{align*}

    Therefore, by Lemma~\ref{lem: IH stage 1} we know
    \begin{align*}
        b_{t+1}&=b_t-\frac{\eta d}{n}(b_t-\frac{1}{n}),\\
        \gamma_{t+1}&\le \gamma_t + O(\eta\sqrt{sd/n}) + \eta\frac{d}{n} O(s\delta+(d/\sqrt{n})\lambda\alpha^2)
        =\gamma_t + \eta O(\sqrt{sd/n}+ds\delta/n).
    \end{align*}
    
    This implies
    \begin{align*}
        b_t &= (1-(1-\eta d/n)^t)/n\le 1/n,\\
        \gamma_t &\le O((\sqrt{sd/n}+ds\delta/n)\eta t) = O(\sigma\sqrt{n/d}+B_\xi).
    \end{align*}
\end{proof}

\lemstageonevdynamic*
\begin{proof}
    We write the update of $a_t$ and $\vDelta_v^\supt{t}$ using the update of $\vv$
    \begin{align*}
        a_{t+1} \mX^\top\vxi + \vDelta_v^\supt{t+1}
        =& \vv^\supt{t+1} - \vv_S^\supt{t+1}
        = \vv^\supt{t} - \vv_S^\supt{t} - \eta \left(\frac{1}{n}\mX^\top\vr^\supt{t}\right)_e\\
        =& a_{t} \mX^\top\vxi + \vDelta_v^\supt{t} - \eta \left(\frac{1}{n}\mX^\top\vr^\supt{t}\right)_e.
    \end{align*}
    For $\left(\frac{1}{n}\mX^\top\vr^\supt{t}\right)_e$, using the decomposition of $\mX^\top\mX\vv/n$ in Lemma~\ref{lem: stage 1 xxv dynamic}, we have
    \begin{align*}
        &\left(\frac{1}{n}\mX^\top\vr^\supt{t}\right)_e\\
        =&\left(\frac{1}{n}\mX^\top\mX\vv^\supt{t} 
        + \frac{1}{n}\mX^\top\mX(\lambda \vw_{S_+}^{\supt{t}\odot 2} - \lambda \vu_{S_-}^{\supt{t}\odot 2}-\vbeta^*) - \frac{1}{n}\mX^\top\vxi+\frac{1}{n}\mX^\top(\lambda \mX\vw_{e_+}^{\supt{t}\odot2}-\lambda \mX\vu_{e_-}^{\supt{t}\odot2})\right)_e\\
        =&\left(\frac{d}{n}\vv_S^\supt{t} 
        + \frac{1}{n}\mX^\top\mX(\lambda \vw_{S_+}^{\supt{t}\odot 2} - \lambda \vu_{S_-}^{\supt{t}\odot 2}-\vbeta^*) +(b_t- \frac{1}{n})\mX^\top\vxi+ \vGamma_t+\frac{1}{n}\mX^\top(\lambda \mX\vw_{e_+}^{\supt{t}\odot2}-\lambda \mX\vu_{e_-}^{\supt{t}\odot2})\right)_e\\
        =&(b_t- \frac{1}{n})(\mX^\top\vxi)_e+\left(
        (\frac{1}{n}\mX^\top\mX-\mI)(\lambda \vw_{S_+}^{\supt{t}\odot 2} - \lambda \vu_{S_-}^{\supt{t}\odot 2}-\vbeta^*) + \vGamma_t+\frac{1}{n}\mX^\top(\lambda \mX\vw_{e_+}^{\supt{t}\odot2}-\lambda \mX\vu_{e_-}^{\supt{t}\odot2})\right)_e.
    \end{align*}

    Therefore, we have the update of $a_t$ and $\zeta_t$ by using Assumption~\ref{assump: 1}, Lemma~\ref{lem: rip approx} and Lemma~\ref{lem: IH stage 1}
    \begin{align*}
        a_{t+1}=&a_t - \eta (b_t - 1/n),\\
        \zeta_{t+1}&\le \zeta_t + \eta O(|nb_t-1| B_\xi+s\delta+\sigma\sqrt{n/d}+B_\xi+(d/\sqrt{n})\lambda\alpha^2).
    \end{align*}
    This implies
    \begin{align*}
        a_t &= \eta t/n - \eta\sum_{\tau<t} b_\tau =(1-(1-\eta d/n)^t)/d\le 1/d\\
        \zeta_t &\le O((B_\xi+s\delta+\sigma\sqrt{n/d})\eta t)
        =O(\sigma\sqrt{n}/d).
    \end{align*}

    Thus, we have $\norm{\vv^\supt{t}-\vv_S^\supt{t}}_2\le a_t O(\sigma\sqrt{dn})+\zeta_t\sqrt{d}=O(\sigma\sqrt{n/d})$. We now bound $\norm{\vv_S}_2$. Since its gradient norm $\norm{\nabla_{\vv_S}L}_2=\norm{(\mX^\top\vr/n)_S}_2\le O(\sqrt{s})$ by Lemma~\ref{lem: IH stage 1} and Assumption~\ref{assump: 1}, we can bound $\norm{\vv_S}_2$ as $\norm{\vv_S}_2=\eta\sum_{\tau\le t}\norm{\nabla_{\vv_S} L^\supt{\tau}}_2\le O(\sqrt{s}\eta t)=O(\sqrt{s}(n/d)\log^2(d)(B_\xi+\sigma \sqrt{n/d}))$. This also implies $\norm{\vv}_2=O(\sigma\sqrt{n/d})$.
\end{proof}

\lemstageonegradapprox*
\begin{proof}
    By Lemma~\ref{lem: stage 1 xxv dynamic} and Lemma~\ref{lem: stage 1 v dynamic} and the choice of parameter, the result directly follows from Lemma~\ref{lem: grad approx}.
\end{proof}

\lemstageonefirstgrow*
\begin{proof}
    For $t\le \tldT_1$, by Lemma~\ref{lem: stage 1 grad approx}, we have for $k\in S_+$ (note that $(\vu_{S_-})_k=0$ in this case. The case $k\in S_-$ is similar, we omit for simplicity)
    \begin{align*}
        w_k^{(t+1)} =& \left(1 -  2\eta\lambda \left(\frac{d}{n}v_k^\supt{t}+\lambda (w_k^\supt{t})^2- \beta_k^*\pm O(B_\xi+\sigma\sqrt{n/d}+s\delta)\right)\right)w_k^\supt{t},\\
        v_k^{(t+1)} =& v_k^\supt{t} -  \eta \left(\frac{d}{n}v_k^\supt{t}+\lambda (w_k^\supt{t})^2- \beta_k^*\pm O(B_\xi+\sigma\sqrt{n/d}+s\delta)\right).
    \end{align*}
    Since $\norm{\vv_S^\supt{t}}_\infty=O(\sqrt{s}(n/d)\log^2(d)(B_\xi+\sigma\sqrt{n/d}))$ by Lemma~\ref{lem: stage 1 v dynamic}, this implies that $(d/n)\norm{\vv_S^\supt{t}}_\infty<\betamin/4$. Thus,
    \begin{align*}
        \lambda (w_k^{(t+1)})^2 
        =& \left(1 -  2\eta\lambda \left(\frac{d}{n}v_k^\supt{t}+\lambda (w_k^\supt{t})^2- \beta_k^*\pm O(B_\xi+\sigma\sqrt{n/d}+s\delta)\right)\right)^2\lambda (w_k^\supt{t})^2\\
        \ge& \left(1 -  2\eta\lambda \left(\lambda (w_k^\supt{t})^2- \beta_k^*/2 \right)\right)^2\lambda (w_k^\supt{t})^2.
    \end{align*}
    
    Therefore, by Lemma~\ref{lem: grouth rate stage 1} within time $O(\log(1/\lambda\alpha^2)/\eta\lambda\betamin)$ we have $\lambda (w_k^\supt{t})^2\ge\betamin/4$ and will remain for $t\le \tldT_1$.
\end{proof}

\lemstageonedecreasehalf*
\begin{proof}
    For $t\le t_0+T'$, by Lemma~\ref{lem: stage 1 grad approx}, we have for $k\in S_+$ (note that $(\vu_{S_-})_k=0$ in this case. The case $k\in S_-$ is similar, we omit for simplicity)
    \begin{align*}
        w_k^{(t+1)} =& \left(1 -  2\eta\lambda \left(\frac{d}{n}v_k^\supt{t}+\lambda (w_k^\supt{t})^2- \beta_k^*\pm \left(B_s+s\delta\norm{\frac{d}{n}\vv_S^\supt{t}+\lambda (\vw_{S_+}^\supt{t})^2-\lambda (\vu_{S_-}^\supt{t})^2 - \vbeta^*}_\infty\right)\right)\right)w_k^\supt{t},\\
        v_k^{(t+1)} =& v_k^\supt{t} -  \eta \left(\frac{d}{n}v_k^\supt{t}+\lambda (w_k^\supt{t})^2- \beta_k^*\pm \left(B_s+s\delta\norm{\frac{d}{n}\vv_S^\supt{t}+\lambda (\vw_{S_+}^\supt{t})^2-\lambda (\vu_{S_-}^\supt{t})^2 - \vbeta^*}_\infty\right)\right).
    \end{align*}
    We claim $\norm{\frac{d}{n}\vv_S^\supt{t}+\lambda (\vw_{S_+}^\supt{t})^2-\lambda (\vu_{S_-}^\supt{t})^2 - \vbeta^*}_\infty\le \norm{\frac{d}{n}\vv_S^\supt{t_0}+\lambda (\vw_{S_+}^\supt{t_0})^2-\lambda (\vu_{S_-}^\supt{t_0})^2 - \vbeta^*}_\infty=B_0$ for $t_0\le t\le t_0+T'$. We show this by induction. At $t=t_0$ it holds. Suppose before $t$ the claim holds. For time $t+1$,
    \begin{align*}
        \frac{d}{n}v_k^\supt{t+1}+\lambda (w_k^\supt{t+1})^2
        =& \frac{d}{n}v_k^\supt{t} 
        - \frac{d}{n} \eta \left(\frac{d}{n}v_k^\supt{t}+\lambda (w_k^\supt{t})^2- \beta_k^*\pm B_0/3\right)\\
        &+ \left(1 -  2\eta\lambda \left(\frac{d}{n}v_k^\supt{t}+\lambda (w_k^\supt{t})^2- \beta_k^*\pm B_0/3\right)\right)^2\lambda (w_k^\supt{t})^2\\
        \ge& \frac{d}{n}v_k^\supt{t} + \lambda (w_k^\supt{t})^2
        - \eta \left(\frac{d}{n}v_k^\supt{t}+\lambda (w_k^\supt{t})^2- \beta_k^*\pm B_0/3\right)\left(\frac{d}{n}+4\lambda^2(w_k^\supt{t})^2\right).
    \end{align*}
    This implies for $t\le t_0+T'$
    \begin{align*}
        \frac{d}{n}v_k^\supt{t+1}+\lambda (w_k^\supt{t+1})^2 - \beta_k^*
        \ge& \left(\frac{d}{n}v_k^\supt{t} + \lambda (w_k^\supt{t})^2 - \beta_k^*\right)
        \left(1- \frac{\eta}{3} \left(\frac{d}{n}+4\lambda^2(w_k^\supt{t})^2\right)\right)\\
        \ge& \left(\frac{d}{n}v_k^\supt{t} + \lambda (w_k^\supt{t})^2 - \beta_k^*\right)
        \left(1- \Omega(\eta \lambda\betamin)\right),
    \end{align*}
    where in the last line we use Lemma~\ref{lem: stage 1 first grow}.
    Thus, if $\frac{d}{n}v_k^\supt{t} + \lambda (w_k^\supt{t})^2 - \beta_k^*<-B_0/2$, then it will increase so that $|\frac{d}{n}v_k^\supt{t} + \lambda (w_k^\supt{t})^2 - \beta_k^*|\le B_0$. Similarly, one can show that if $\frac{d}{n}v_k^\supt{t} + \lambda (w_k^\supt{t})^2 - \beta_k^*>B_0/2$, then it will decrease so that $|\frac{d}{n}v_k^\supt{t} + \lambda (w_k^\supt{t})^2 - \beta_k^*|\le B_0$. In this way, we finish the induction.
    
    Moreover, from the above calculations, we can see that if $\frac{d}{n}v_k^\supt{t} + \lambda (w_k^\supt{t})^2 - \beta_k^*<-B_0/2$, then within time $O(1/\eta \lambda\betamin)$, $|\frac{d}{n}v_k^\supt{t} + \lambda (w_k^\supt{t})^2 - \beta_k^*|\le B_0/2$. Similarly, if $\frac{d}{n}v_k^\supt{t} + \lambda (w_k^\supt{t})^2 - \beta_k^*>B_0/2$, then within time $O(1/\eta \lambda\betamin)$, $|\frac{d}{n}v_k^\supt{t} + \lambda (w_k^\supt{t})^2 - \beta_k^*|\le B_0/2$. By Lemma~\ref{lem: stability}, we know once $|\frac{d}{n}v_k^\supt{t} + \lambda (w_k^\supt{t})^2 - \beta_k^*| \le B_0/2$, it will remain bounded by $B_0/2$. Therefore, we know $T'=O(1/\eta\lambda\betamin)$.
\end{proof}

\lemstability*
\begin{proof}
    By Lemma~\ref{lem: stage 1 grad approx}, we have for $k\in S_+$ (note that $(\vu_{S_-})_k=0$ in this case. The case $k\in S_-$ is similar, we omit for simplicity)
    \begin{align*}
        w_k^{(t+1)} =& \left(1 -  2\eta\lambda \left(\frac{d}{n}v_k^\supt{t}+\lambda (w_k^\supt{t})^2- \beta_k^*\pm (B_s+s\delta\mu)\right)\right)w_k^\supt{t},\\
        v_k^{(t+1)} =& v_k^\supt{t} -  \eta \left(\frac{d}{n}v_k^\supt{t}+\lambda (w_k^\supt{t})^2- \beta_k^*\pm (B_s+s\delta\mu)\right).
    \end{align*}
    Since $\lambda (w_k^\supt{t})^2=\betamin/4$ by Lemma~\ref{lem: stage 1 first grow}, we have
    \begin{align*}
        \frac{d}{n}v_k^\supt{t+1}+\lambda (w_k^\supt{t+1})^2
        =& \frac{d}{n}v_k^\supt{t} 
        - \frac{d}{n} \eta \left(\frac{d}{n}v_k^\supt{t}+\lambda (w_k^\supt{t})^2- \beta_k^*\pm (B_s+s\delta\mu)\right)\\
        &+ \left(1 -  2\eta\lambda \left(\frac{d}{n}v_k^\supt{t}+\lambda (w_k^\supt{t})^2- \beta_k^*\pm (B_s+s\delta\mu)\right)\right)^2\lambda (w_k^\supt{t})^2\\
        \ge& \frac{d}{n}v_k^\supt{t} + \lambda (w_k^\supt{t})^2
        - \eta \left(\frac{d}{n}v_k^\supt{t}+\lambda (w_k^\supt{t})^2- \beta_k^*\pm \underbrace{(B_s+s\delta\mu)}_{=:\textrm{err}}\right)\left(\frac{d}{n}+4\lambda^2(w_k^\supt{t})^2\right).
    \end{align*}
    This implies for $t\ge t_0$, if $\frac{d}{n}v_k^\supt{t}+\lambda (w_k^\supt{t})^2- \beta_k^* < -2\textrm{err}$, we have
    \begin{align*}
        \frac{d}{n}v_k^\supt{t+1}+\lambda (w_k^\supt{t+1})^2 - \beta_k^*
        \ge& \left(\frac{d}{n}v_k^\supt{t} + \lambda (w_k^\supt{t})^2 - \beta_k^*\right)
        \left(1- \frac{\eta}{2} \left(\frac{d}{n}+4\lambda^2(w_k^\supt{t})^2\right)\right)\\
        \ge& \left(\frac{d}{n}v_k^\supt{t} + \lambda (w_k^\supt{t})^2 - \beta_k^*\right)
        \left(1- \Omega(\eta \lambda\betamin)\right)
    \end{align*}
    Thus, $\frac{d}{n}v_k^\supt{t} + \lambda (w_k^\supt{t})^2 - \beta_k^*$ will increase in this case. Therefore, we know $\frac{d}{n}v_k^\supt{t} + \lambda (w_k^\supt{t})^2 - \beta_k^*\ge-\max\{\mu,2\textrm{err}\}=-\max\{\mu,2(B_s+s\delta\mu)\}$ for all $t\ge t_0$. Similarly, given $\eta$ is small enough, we can also get a similar upper bound. Thus, we finish the proof. 
\end{proof}

\lemstageonetime*
\begin{proof}
    We can first use Lemma~\ref{lem: stage 1 first grow} and then repeatedly using Lemma~\ref{lem: stage 1 decrease half} $\log(1/4B_s)$ times. We get within time $O(\log(1/\lambda\alpha^2(B_\xi+\sigma\sqrt{n/d}))/\eta\lambda\betamin)=O(\log(1/\alpha)/\eta\lambda\betamin)$, $\norm{\frac{d}{n}\vv_S^\supt{t}+\lambda (\vw_{S_+}^\supt{t})^2-\lambda (\vu_{S_-}^\supt{t})^2 - \vbeta^*}_\infty\le 4B_s=C_{T_1}(B_\xi+\sigma\sqrt{n/d})$.
\end{proof}

\subsection{Technical Lemmas}
In this subsection, we collect several technical lemmas that are used in the proof.

\begin{lemma}\label{lem: grouth rate stage 1}
    Suppose $z_{t+1}=(1-\eta(z_t-\mu))^2z_t$ with $\eta,\mu,z_0>0$ and $z_0\le \mu-\eps$. Then if $\eta\le \mu/2$, within time $T=O((1/\eta\mu)(\log(\mu/z_0)+\log(\mu/\eps)))$ we have $|z_T-\mu|\le\eps$. Moreover, we have $|z_t-\mu|\le\eps$ for $t\ge T$.
\end{lemma}
\begin{proof}
    Denote $T_1:=\inf\{t:z_t\ge\mu/2\}$ and $T_2:=\inf\{t:|z_t-\mu|\le \eps\}$. We bound $T_1$ and $T_2-T_1$ respectively in below.
    
    For $t\le T_1$, we have $z_{t+1}\ge (1+\eta\mu/2)^2z_t\ge (1+\eta\mu/2)^{2t}z_0$. Therefore, $T_1=O((1/\eta\mu)\log(\mu/z_0))$.
    For $T_1\le t\le T_2$, we have $z_{t+1}\ge z_t -2\eta(z_t-\mu)z_t\ge z_t-\eta(z_t-\mu)\mu$. This implies $z_{t+1}-\mu\ge(1-\eta\mu)(z_t-\mu)\ge(1-\eta\mu)^{t-T_1}(z_{T_1}-\mu)$. Therefore, $T_2-T_2=O((1/\eta\mu)\log(\mu/\eps))$.
    Together we know $T=T_1+T_2=O((1/\eta\mu)(\log(\mu/z_0)+\log(\mu/\eps)))$. 
    
    We then show once $|z_t-\mu|\le\eps$, it will stay close to $\mu$. To see this, if $-\eps\le z_t-\mu<0$, then from the above calculation we know $z_{t+1}-\mu\ge(1-\eta\mu)(z_t-\mu)\ge-\eps$. If $0\le z_t-\mu\le\eps$, then $z_{t+1}=(1-\eta(z_t-\mu))^2z_t\le z_t\le\mu+\eps$. Therefore, we know $|z_t-\mu|\le\eps$ for $t\ge T_1$.
\end{proof}

\begin{lemma}\label{lem: odot bound}
    For $\alpha,\beta\in \R^{d}$, we have $\norm{\alpha\odot\beta}_2\le \norm{\alpha}_2\norm{\beta}_\infty$, $\norm{\alpha^{\odot k}}_2\le\norm{\alpha}_2^k$ for $k\ge 1$.
\end{lemma}
\begin{proof}
    We have
    \begin{align*}
        \norm{\alpha\odot\beta}_2^2
        =\sum_i \alpha_i^2\beta_i^2
        \le \norm{\alpha}_2^2\norm{\beta}_\infty^2,\\
        \norm{\alpha^{\odot k}}_2^2 = \sum_i\alpha_i^{2k}=\norm{\alpha}_{2k}^{2k}
        \le \norm{\alpha}_2^{2k}.
    \end{align*}
\end{proof}

\section{Proof for Stage 2}\label{sec: pf stage 2}

In Stage 2, we will show that the training loss goes to $\eps$ while the test loss $\norm{\vbeta-\vbeta^*}_2$ remains small. In particular, we will split into 2 sub-stages: in Stage 2.1, train loss decreases to $\norm{\vr}_2=O(\sigma)$ (Lemma~\ref{lem: stage 2.1}), and in Stage 2.2 we use a NTK-type analysis (Lemma~\ref{lem: stage 2.2}). Note that it suffices to combine Lemma~\ref{lem: stage 2.1} and Lemma~\ref{lem: stage 2.2} to get Lemma~\ref{lem: stage 2}.

Throughout Stage 2, we mostly rely on $\vv_e$ to fit the noise in order to reduce the loss; at the same time, we show that the variables used in Stage 1 continue to fit the signal and all the other variables remain small. This can be done by an NTK-type analysis when the loss is very small. However, for the first part of Stage 2 we still need to track the dynamics of $\vv$ and $\mX^\top\mX\vv$ carefully.

\subsection{Stage 2.1: train loss decreases to $\norm{r}_2=O(\sigma)$}
Our goal in this stage is to show that the loss decreases to $O(\sigma^2)$ and that the non-signal entries remain small. We formalize this in the following main lemma.
\begin{restatable}[Stage 2.1]{lemma}{lemstagetwoone}\label{lem: stage 2.1}
Let $T_{21}:=\inf\{t:\norm{\vr^\supt{t}}_2\le C_{T_{21}}\sigma\}$ with large enough universal constant $C_{T_{21}}$. Then, we have $T_{21}-T_1=O((n/\eta d)\log(n))$ and the following hold with large enough universal constant $C_{21}$:
\begin{itemize}
    \item $\norm{\frac{d}{n}\vv_S^\supt{T_{21}}\lambda \vw_{S_+}^{\supt{T_{21}}\odot2}-\lambda\vu_{S_-}^{\supt{T_{21}}\odot2}-\vbeta^*}_\infty\le C_{21}(B_\xi+\sigma\sqrt{n/d})$
    \item $\norm{\vw_{e_+}^\supt{T_{21}}}_\infty, \norm{\vu_{e_-}^\supt{T_{21}}}_\infty C_{21}\alpha$.
    \item $\norm{\vv^\supt{T_{21}}}_2\le C_{21}\sigma \sqrt{n/d}$ and $\norm{\vv_S^\supt{T_{21}}}_2\le C_{21}\sqrt{s}(n/d)\log^2(d)(B_\xi+\sigma\sqrt{n/d})$.
\end{itemize}
\end{restatable}

To prove this, we will maintain the following inductive hypothesis, which shows the non-signal entries remain small. The overall strategy is to show that entries of $\vv$ will allow us to fit the noise and hence reduce loss, and we do this by using a similar strategy to track the dynamics of $\vv$ as in Stage 1. 
\begin{restatable}[Inductive Hypothesis for Stage 2.1]{lemma}{lemIHstagetwoone}\label{lem: IH stage 2.1}
    For $T_1\le t\le \tldT_{21}:=T_1+C_{\tldT_{21}}(n\log(n)/\eta d)$ with a large enough universal constant $C_{\tldT_{21}}$, we have the following hold with large enough universal constant $\tldC_{21}$:
    \begin{itemize}
        \item $\norm{\frac{d}{n}\vv_S^\supt{t}+\lambda \vw_{S_+}^{\supt{t}\odot2}-\lambda\vu_{S_-}^{\supt{t}\odot2}-\vbeta^*}_\infty\le \tldC_{21}(B_\xi+\sigma\sqrt{n/d})$
        \item $\norm{\vw_{e_+}^\supt{t}}_\infty, \norm{\vu_{e_-}^\supt{t}}_\infty \le \tldC_{21} \alpha$.
        \item $\norm{\vv_S^\supt{t}}_2\le\tldC_{21}\sqrt{s}(n/d)\log^2(d)(B_\xi+\sigma\sqrt{n/d})$.
        \item $\norm{\vr^\supt{t}}_2=(1-\Omega(\eta d/n))^{t-T_1}\cdot C_1\sigma\sqrt{n}$.
    \end{itemize}
    In particular, the first point and third point imply that $\norm{\lambda \vw_{S_+}^{\supt{t}\odot2}-\lambda\vu_{S_-}^{\supt{t}\odot2}-\vbeta^*}_\infty=2\tldC_{21}\sqrt{s}(B_\xi+\sigma\sqrt{n/d})\log^2(d)$. 
    The last point implies that $T_{21}-T_1=O((n/\eta d)\log(n))$. Moreover, by the choice of parameters, $O(d/\sqrt{n}) \lambda \norm{\vw_{e_+}^\supt{t}}_\infty^2=O(B_\xi/\log d)$, $O(d/\sqrt{n}) \lambda \norm{\vu_{e_-}^\supt{t}}_\infty^2=O(B_\xi/\log d)$. 
\end{restatable}

Similar as Stage 1, we discuss the constant dependency here. All the constants in big-$O$ in this subsection, except main result Lemma~\ref{lem: stage 2.1}, should only depends on universal constants $C_{T_{21}},\tldC_{21},C_{\tldT_{21}}$ as well as the constants in Stage 1. To ensure this, we especially need to choose the constant in $\lambda=\Theta\left(d\sigma^{-1} n^{-1}(\sqrt{\log (d)/n}+\sqrt{n/d})^{-1}\log^{-1}(n)\right)$ to be small enough to ensure the nonsignal entries do not grow large. See the proof of Lemma~\ref{lem: IH stage 2.1} for details.

\subsubsection{Dynamics of $v$}\label{subsec: appendix stage 2.1 v dynamic}

As in Stage 1, we analyze the decomposition of $\mX^\top\mX\vv/n$ and $\vv$ separately. The proofs are very similar to Lemma~\ref{lem: stage 1 xxv dynamic} and Lemma~\ref{lem: stage 1 v dynamic} in Stage 1, but several terms will now have a tighter bound. We defer the proofs to Appendix~\ref{subsec: appendix stage 2.1 pf}.

For the decomposition of $\mX^\top\mX\vv/n$  we have
\begin{restatable}{lemma}{lemstagetwoonexxvdynamic}\label{lem: stage 2.1 xxv dynamic}
    Recall the decomposition in \eqref{eq: decomp xxv}
    \begin{align*}
        \frac{1}{n}\mX^\top\mX\vv^\supt{t}
        &= \frac{d}{n} \vv_S^\supt{t} + b_t (\mX^\top\vxi)_e + \vGamma_t,\\
        b_{t+1}&=b_t-\frac{\eta d}{n}(b_t-\frac{1}{n}),
    \end{align*}
    where $\norm{\vGamma^\supt{t}}_\infty\le \gamma_t$ and recall the notation $\vbeta_S=\sum_{i:\beta^*_i\ne 0}\beta_i\ve_i$, $\vbeta_e=\sum_{i:\beta^*_i= 0}\beta_i\ve_i$.
    Suppose Lemma~\ref{lem: IH stage 2.1} holds. We have for $T_1\le t\le \tldT_{21}$
    \begin{align*}
        b_t &= (1-(1-\eta d/n)^t)/n\le 1/n,\\
        \gamma_t &\le \gamma_{T_1} + O(\sigma\sqrt{d/n}+(d B_\xi/n\log d)\eta t) = O(\sigma\sqrt{n/d}+B_\xi).
    \end{align*}
\end{restatable}

For the decomposition of $\vv$ we have
\begin{restatable}{lemma}{lemstagetwoonevdynamic}\label{lem: stage 2.1 v dynamic}
    Recall the decomposition in \eqref{eq: decomp v}
    \begin{align*}
        \vv^\supt{t}&= \vv_S^\supt{t} + a_t \mX^\top\vxi + \vDelta_v^\supt{t},\\
        a_{t+1}&=a_t - \eta (b_t - 1/n),
    \end{align*}
    where $\norm{\vDelta_v^\supt{t}}_\infty\le \zeta_t$. and recall the notation $\vbeta_S=\sum_{i:\beta^*_i\ne 0}\beta_i\ve_i$, $\vbeta_e=\sum_{i:\beta^*_i= 0}\beta_i\ve_i$. 
    Suppose Lemma~\ref{lem: IH stage 2.1} holds. We have for $T_1\le t\le \tldT_{21}$
    \begin{align*}
        a_t&=(1-(1-\eta d/n)^t)/d\le 1/d\\
        \zeta_t&= \zeta_{T_1} + O((B_\xi+\sigma\sqrt{n/d})\eta (t-T_1))=O((B_\xi+\sigma\sqrt{n/d})n\log(n)/d).
    \end{align*}
    In particular, we can show that $\norm{\vv^\supt{t}}_2=O(\sigma\sqrt{n/d})$.
\end{restatable}

\subsubsection{Implications of Inductive Hypothesis Lemma~\ref{lem: IH stage 2.1}}\label{subsec: IH stage 2.1 implications}

Given the dynamics of $\vv$, we now have the approximation of gradient by Lemma~\ref{lem: grad approx}.
\begin{restatable}{lemma}{lemstagetwogradapprox}\label{lem: stage 2 grad approx}
    In the setting of Lemma~\ref{lem: stage 2.1 xxv dynamic} and Lemma~\ref{lem: stage 2.1 v dynamic}, we have for $T_1\le t\le \tldT_{21}$
    \begin{align*}
        \nabla_\vw L &= \left(\frac{1}{n} \mX^\top \vr \right)\odot(2\lambda \vw)=2\lambda (\frac{d}{n} \vv_S + \lambda \vw^{\odot 2}_{S_+} - \lambda \vu^{\odot2}_{S_-}-\vbeta^*+\vDelta_r)\odot \vw,\\
        \nabla_\vu L &= -\left(\frac{1}{n} \mX^\top \vr \right)\odot(2\lambda \vu)=-2\lambda (\frac{d}{n} \vv_S + \lambda \vw^{\odot 2}_{S_+} - \lambda \vu^{\odot 2}_{S_-}-\vbeta^*+\vDelta_r)\odot \vu,\\
        \nabla_v L &= \frac{1}{n} \mX^\top \vr,
    \end{align*}
    where 
    \begin{align*}
        \norm{\vDelta_r^\supt{t}}_\infty
        =& O\left(B_\xi+\sigma\sqrt{n/d}\right)
        +s\delta\norm{\frac{d}{n}\vv_S^\supt{t}+\lambda \vw_{S_+}^{\supt{t}\odot 2} - \lambda \vu_{S_-}^{\supt{t}\odot 2}-\vbeta^*}_\infty.
    \end{align*} 
\end{restatable}

\subsubsection{Proof of Inductive Hypothesis Lemma~\ref{lem: IH stage 2.1} and Lemma~\ref{lem: stage 2.1}}
Now we are ready to prove the induction hypothesis for Stage 2.1 and Lemma~\ref{lem: stage 2.1}.
\lemIHstagetwoone*
\begin{proof}
    We show these inductively on $t$. For $t=T_1$, we know it holds by Lemma~\ref{lem: stage 1}. Suppose it holds before time $t$, then at time $t+1$ we will show it still hold. 

    For $\norm{\frac{d}{n} \vv_S^\supt{t+1}+ \lambda \vw_{S_+}^{\supt{t+1}\odot 2} - \lambda \vu_{S_-}^{\supt{t+1}\odot 2}-\vbeta^*}_\infty$, let $k\in S_+$ (the case $k\in S_-$ can be handled similarly, we omit for simplicity). 
    Since by the choice of parameter $(d/n)\norm{\vv_S^\supt{t}}_\infty<\betamin/2$,  
    we know $\lambda(w_k^\supt{t})^2=\betamin/4$. For $T_1\le t\le \tldT_{21}$, by Lemma~\ref{lem: stability} and Lemma~\ref{lem: stage 2 grad approx}, we know $\norm{\frac{d}{n}\vv^\supt{t+1}+\lambda \vw^{\supt{t+1}\odot2}-\lambda \vu^{\supt{t+1}\odot2} -\vbeta^*}_\infty=\tldC_{21}(B_\xi+\sigma\sqrt{n/d})$ with large enough $\tldC_{21}$.
    
    For $k\not\in S$, consider $w_k$ ($u_k$ can be bounded similarly), we have the dynamics by Lemma~\ref{lem: stage 2 grad approx}
    \begin{align*}
        w_k^{(t+1)} \le& \left(1 +  2\eta\lambda O(B_\xi+\sigma\sqrt{n/d})\right)w_k^\supt{t}.
    \end{align*}
    This means $|w_k^\supt{t}|=(1 + O(\eta\lambda(B_\xi+\sigma\sqrt{n/d}))^{t-T_{1}} \cdot C_1\alpha $. To show $|w_k^\supt{t}|$ remain as $O(\alpha)$, recall the choice of $\lambda=\Theta\left(d\sigma^{-1} n^{-1}(\sqrt{\log (d)/n}+\sqrt{n/d})^{-1}\log^{-1}(n)\right)$ and $\tldT_{21}-T_1\le C_{\tldT_{21}}n\log(n)/\eta d$, we only need to choose the constant in $\lambda$ to be small enough. In this way, we get $|w_k^\supt{t}|\le \tldC_{21}\alpha$ with large enough constant $\tldC_{21}$.

    It remains to consider $w_k$ with $k\in S_-$ and $u_k$ with $k\in S_+$. We will focus on $w_k$ with $k\in S_-$, the other follows the same calculation. Similar in the proof of Lemma~\ref{lem: IH stage 1}, we have
    \begin{align*}
        w_k^\supt{t+1} u_k^\supt{t+1}
        =& \left(1-2\eta\lambda\left(\frac{1}{n}\mX^\top \vr^\supt{t}\right)_k\right)w_k^\supt{t} 
        \cdot \left(1+2\eta\lambda\left(\frac{1}{n}\mX^\top \vr^\supt{t}\right)_k\right)u_k^\supt{t}
        \le w_k^\supt{t} u_k^\supt{t}\le \alpha^2.
    \end{align*}
    We know $u_k^\supt{t}\ge\alpha$. This implies that $|w_k^\supt{t}|\le \alpha$.

    For $\norm{\vv_S}_2$, we have by Lemma~\ref{lem: stage 2 grad approx} and Lemma~\ref{lem: stage 1}
    \begin{align*}
        \norm{\vv_S^\supt{t+1}}_2
        \le&\norm{\vv_S^\supt{t}}_2 + \eta \norm{\frac{1}{n}(\mX^\top\vr^\supt{t})_S}_2  
        \le \norm{\vv_S^\supt{T_1}}_2 + O(\sqrt{s}(B_\xi+\sigma\sqrt{n/d})\eta(t-T_1) )\\
        =&O(\sqrt{s}(n/d)\log^2(d)(B_\xi+\sigma\sqrt{n/d}))\\
        \le&\tldC_{21}\sqrt{s}(n/d)\log^2(d)(B_\xi+\sigma\sqrt{n/d})
    \end{align*}
    with large enough constant $\tldC_{21}$. 
    
    For the bound on $\norm{\vr^\supt{t+1}}_2$, using the same calculation as in the proof of Lemma~\ref{lem: IH stage 1}, we can show it is true.
\end{proof}

Given the above induction hypothesis, we are ready to prove the main result for Stage 2.1.
\lemstagetwoone*
\begin{proof}
    The first two points and the bound on $T_{21}-T_1$ follow from Lemma~\ref{lem: IH stage 2.1}. The last point follow from Lemma~\ref{lem: stage 2.1 v dynamic} and Lemma~\ref{lem: IH stage 2.1}. As mentioned earlier, we can choose large enough universal constant $C_{21}$ to serve as upper bound, since all big-$O$ here only hide constants depend on universal constants $C_{T_{21}},\tldC_{21},C_{\tldT_{21}}$ as well as the constants in Stage 1.
\end{proof}

\subsubsection{Omitted Proofs in Section~\ref{subsec: appendix stage 2.1 v dynamic} and Section~\ref{subsec: IH stage 2.1 implications}}\label{subsec: appendix stage 2.1 pf}
In this subsection, we give the proof of Lemma~\ref{lem: stage 2.1 xxv dynamic}, Lemma~\ref{lem: stage 2.1 v dynamic} and Lemma~\ref{lem: stage 2 grad approx}.

\lemstagetwoonexxvdynamic*
\begin{proof}
    The proof here is almost the same as in the proof of Lemma~\ref{lem: stage 1 xxv dynamic} in Stage 1. The only difference is that we know have better bounds on the error terms.
    We first write the update of $b_t$ and $\vGamma_t$ using the update of $\vv$.
    \begin{align*}
        b_{t+1} (\mX^\top\vxi)_e + \vGamma_{t+1} 
        =& \frac{1}{n}\mX^\top\mX\vv^\supt{t+1} -  \frac{d}{n} \vv_S^\supt{t+1} \\
        =& \frac{1}{n}\mX^\top\mX\vv^\supt{t} - \frac{d}{n} \vv_S^\supt{t} 
        - \eta \frac{1}{n}\mX^\top\mX \frac{1}{n}\mX^\top \vr^\supt{t} 
        + \eta\frac{d}{n} \left(\frac{1}{n}\mX^\top\vr^\supt{t}\right)_S\\
        =& b_{t} (\mX^\top\vxi)_e + \vGamma_{t} - \frac{\eta}{n^2}\mX^\top\mX \mX^\top \vr^\supt{t} 
        + \eta\frac{d}{n} \left(\frac{1}{n}\mX^\top\vr^\supt{t}\right)_S\\
        =& b_{t} (\mX^\top\vxi)_e + \vGamma_{t} - \frac{\eta}{n^2}\mX^\top(\mX \mX^\top -d\mI)\vr^\supt{t} 
        - \eta\frac{d}{n} \left(\frac{1}{n}\mX^\top\vr^\supt{t}\right)_e.
    \end{align*}

    We bound the last two terms one by one. For $\frac{\eta}{n^2}\mX^\top(\mX \mX^\top -d\mI)\vr^\supt{t}$, we have by Assumption~\ref{assump: 1} and Lemma~\ref{lem: IH stage 2.1}
    \begin{align*}
        \norm{\frac{\eta}{n^2}\mX^\top(\mX \mX^\top -d\mI)\vr^\supt{t}}_\infty
        \le \frac{\eta}{n} O(\frac{1}{\sqrt{n}}\cdot \sqrt{dn}) (1-\Omega(\eta d/n))^{t-T_1}O(\sigma\sqrt{n})
        =O(\eta\sigma\sqrt{d/n})(1-\Omega(\eta d/n))^{t-T_1}.
    \end{align*}
    For $\eta\frac{d}{n} \left(\frac{1}{n}\mX^\top\vr^\supt{t}\right)_e$, we have
    \begin{align*}
        &\left(\frac{1}{n}\mX^\top\vr^\supt{t}\right)_e\\
        =&\left(\frac{1}{n}\mX^\top\mX\vv^\supt{t} 
        + \frac{1}{n}\mX^\top\mX(\lambda \vw_{S_+}^{\supt{t}\odot 2} - \lambda \vu_{S_-}^{\supt{t}\odot 2}-\vbeta^*) - \frac{1}{n}\mX^\top\vxi+\frac{1}{n}\mX^\top(\lambda \mX\vw_{e_+}^{\supt{t}\odot2}-\lambda \mX\vu_{e_-}^{\supt{t}\odot2})\right)_e\\
        =&\left(\frac{d}{n}\vv_S^\supt{t} 
        + \frac{1}{n}\mX^\top\mX(\lambda \vw_{S_+}^{\supt{t}\odot 2} - \lambda \vu_{S_-}^{\supt{t}\odot 2}-\vbeta^*) +(b_t- \frac{1}{n})\mX^\top\vxi+ \vGamma_t+\frac{1}{n}\mX^\top(\lambda \mX\vw_{e_+}^{\supt{t}\odot2}-\lambda \mX\vu_{e_-}^{\supt{t}\odot2})\right)_e\\
        =&(b_t- \frac{1}{n})(\mX^\top\vxi)_e+\left(
        (\frac{1}{n}\mX^\top\mX-\mI)(\lambda \vw_{S_+}^{\supt{t}\odot 2} - \lambda \vu_{S_-}^{\supt{t}\odot 2}-\vbeta^*) + \vGamma_t+\frac{1}{n}\mX^\top(\lambda \mX\vw_{e_+}^{\supt{t}\odot2}-\lambda \mX\vu_{e_-}^{\supt{t}\odot2})\right)_e.
    \end{align*}

    Therefore, we know by Lemma~\ref{lem: IH stage 2.1}
    \begin{align*}
        b_{t+1}&=b_t-\frac{\eta d}{n}(b_t-\frac{1}{n}),\\
        \gamma_{t+1}&\le \gamma_t + (1-O(\eta d/n))^{t-T_1}O(\eta\sigma\sqrt{d/n}) + \eta\frac{d}{n} O(B_\xi/\log d+(d/\sqrt{n})\lambda\alpha^2)\\
        &=\gamma_t + (1-O(\eta d/n))^{t-T_1}O(\eta\sigma\sqrt{d/n}) + \eta O(d B_\xi/n\log d).
    \end{align*}
    By Lemma~\ref{lem: stage 1 xxv dynamic}, this implies
    \begin{align*}
        b_t &= (1-\eta d/n)^{t-T_1}b_{T_1}+(1-(1-\eta d/n)^{t-T_1})/n
        =(1-(1-\eta d/n)^t)/n\le 1/n,\\
        \gamma_t &\le \gamma_{T_1} + O(\sigma\sqrt{n/d}+(d B_\xi/n\log d)\eta (t-T_1)) = O(\sigma\sqrt{n/d}+B_\xi).
    \end{align*}
\end{proof}

\lemstagetwoonevdynamic*
\begin{proof}
    The proof here is almost the same as in the proof of Lemma~\ref{lem: stage 1 v dynamic} in Stage 1. The only difference is that we know have better bounds on the error terms.
    We write the update of $a_t$ and $\vDelta_v^\supt{t}$ using the update of $\vv$
    \begin{align*}
        a_{t+1} \mX^\top\vxi + \vDelta_v^\supt{t+1}
        =& \vv^\supt{t+1} - \vv_S^\supt{t+1}
        = \vv^\supt{t} - \vv_S^\supt{t} - \eta \left(\frac{1}{n}\mX^\top\vr^\supt{t}\right)_e\\
        =& a_{t} \mX^\top\vxi + \vDelta_v^\supt{t} - \eta \left(\frac{1}{n}\mX^\top\vr^\supt{t}\right)_e.
    \end{align*}
    For $\left(\frac{1}{n}\mX^\top\vr^\supt{t}\right)_e$, using the decomposition of $\mX^\top\mX\vv/n$ in Lemma~\ref{lem: stage 2.1 xxv dynamic}, we have
    \begin{align*}
        &\left(\frac{1}{n}\mX^\top\vr^\supt{t}\right)_e\\
        =&\left(\frac{1}{n}\mX^\top\mX\vv^\supt{t} 
        + \frac{1}{n}\mX^\top\mX(\lambda \vw_{S_+}^{\supt{t}\odot 2} - \lambda \vu_{S_-}^{\supt{t}\odot 2}-\vbeta^*) - \frac{1}{n}\mX^\top\vxi+\frac{1}{n}\mX^\top(\lambda \mX\vw_{e_+}^{\supt{t}\odot2}-\lambda \mX\vu_{e_-}^{\supt{t}\odot2})\right)_e\\
        =&\left(\frac{d}{n}\vv_S^\supt{t} 
        + \frac{1}{n}\mX^\top\mX(\lambda \vw_{S_+}^{\supt{t}\odot 2} - \lambda \vu_{S_-}^{\supt{t}\odot 2}-\vbeta^*) +(b_t- \frac{1}{n})\mX^\top\vxi+ \vGamma_t+\frac{1}{n}\mX^\top(\lambda \mX\vw_{e_+}^{\supt{t}\odot2}-\lambda \mX\vu_{e_-}^{\supt{t}\odot2})\right)_e\\
        =&(b_t- \frac{1}{n})(\mX^\top\vxi)_e+\left(
        (\frac{1}{n}\mX^\top\mX-\mI)(\lambda \vw_{S_+}^{\supt{t}\odot 2} - \lambda \vu_{S_-}^{\supt{t}\odot 2}-\vbeta^*) + \vGamma_t+\frac{1}{n}\mX^\top(\lambda \mX\vw_{e_+}^{\supt{t}\odot2}-\lambda \mX\vu_{e_-}^{\supt{t}\odot2})\right)_e.
    \end{align*}

    Therefore, we have the update of $a_t$ and $\zeta_t$ by using Lemma~\ref{lem: rip approx}, Assumption~\ref{assump: 1} and Lemma~\ref{lem: IH stage 2.1}
    \begin{align*}
        a_{t+1}&=a_t - \eta (b_t - 1/n),\\
        \zeta_{t+1}&\le \zeta_t + \eta O(|nb_t-1| B_\xi+B_\xi/\log d+\sigma\sqrt{n/d}+B_\xi+(d/\sqrt{n})\lambda\alpha^2).
    \end{align*}
    By Lemma~\ref{lem: stage 1 v dynamic}, this implies
    \begin{align*}
        a_t &= \eta t/n - \eta\sum_{\tau<t} b_\tau =(1-(1-\eta d/n)^t)/d\le 1/d\\
        \zeta_t &\le \zeta_{T_1} + O((B_\xi+\sigma\sqrt{n/d})\eta (t-T_1))=O((B_\xi+\sigma\sqrt{n/d})n\log(n)/d).
    \end{align*}
    
    We now bound $\norm{\vv}_2$. Since its gradient norm $\norm{\nabla_{\vv}L}_2=\norm{\mX^\top\vr/n}_2\le (1-\Omega(\eta d/n))^{t-T_1}O(\sigma\sqrt{d/n})$ by Lemma~\ref{lem: IH stage 2.1} and Assumption~\ref{assump: 1}, we can bound $\norm{\vv^\supt{t}}_2\le\norm{\vv^\supt{T_1}}_2+\eta\sum_{T_1\le\tau\le t}\norm{\nabla_\vv L^\supt{\tau}}_2=\norm{\vv^\supt{T_1}}_2+O(\sigma\sqrt{n/d})=O(\sigma\sqrt{n/d})$.
    
\end{proof}

\lemstagetwogradapprox*
\begin{proof}
    By Lemma~\ref{lem: stage 2.1 xxv dynamic} and Lemma~\ref{lem: stage 2.1 v dynamic} and the choice of parameter, the result directly follows from Lemma~\ref{lem: grad approx}.
\end{proof}

\subsection{Stage 2.2}
After Stage 2.1, the loss is already very small. This allows us to further tighten the bound of several terms and use an NTK-type analysis to show that the parameters do not move much while reduce the training loss to $\eps$.

\begin{lemma}\label{lem: stage 2.2}
    Let $T_{22}:=\inf\{t:L(\vu^\supt{t}, \vw^\supt{t},\vv^\supt{t})=\norm{\vr^\supt{t}}^2/n\le \eps\}$. Then $T_{22}-T_{21}=O(n\log(\sigma/\eps)/\eta d)$ and the following hold:
    \begin{itemize}
        \item $\norm{\frac{d}{n}\vv_S^\supt{T_{22}}+\lambda \vw_{S_+}^{\supt{T_{22}}\odot 2}-\lambda \vu_{S_-}^{\supt{T_{22}}\odot 2}-\vbeta^*}_\infty=O(B_\xi+\sigma\sqrt{n/d})$
        \item $\norm{\vw_{e_+}^\supt{T_{22}}}_\infty,\norm{\vu_{e_-}^\supt{T_{22}}}_\infty=O(\alpha)$
        \item $\norm{\vv^\supt{T_{22}}}_2=O(\sigma\sqrt{n/d})$, $\norm{\vv_S^\supt{T_{22}}}_2=O(\sqrt{s}(n/d)\log^2(d)(B_\xi+\sigma\sqrt{n/d}))$
        \item $\norm{\vr^\supt{t}}_2=(1-\Omega(\eta d/n))^{t-T_{21}} O(\sigma)$
    \end{itemize}
    In particular, the above imply that $\norm{\vbeta^\supt{T_{22}}-\vbeta^*}_2=O(\sqrt{s}\log^2(d)(B_\xi+\sigma\sqrt{n/d}))$. 
    Moreover, for every $t\ge T_{22}$,  the above still hold and train loss $L^\supt{t}\le\eps$.
\end{lemma}
For the constant dependency, the big-$O$ here can be replaced by a large enough universal constant, similar to the argument in Stage 1 and Stage 2.1. We omit here for simplicity.

\begin{proof}
    We show these by induction. At $t=T_{21}$, we know they hold by Lemma~\ref{lem: stage 2.1}. Suppose before time $t$ they hold, then at time $t+1$ we know $\norm{\mX^\top \vr^\supt{\tau}/n}_\infty=(1-\Omega(\eta d/n))^{\tau-T_{21}}O(\sigma/\sqrt{n})$ for any $\tau\le t$ by Assumption~\ref{assump: 1}. 
    
    For $\norm{\frac{d}{n}\vv_S^\supt{t+1}+\lambda \vw_{S_+}^{\supt{t+1}\odot 2}-\lambda \vu_{S_-}^{\supt{t+1}\odot 2}-\vbeta^*}_\infty$, consider the $k$-th entry with $k\in S_+$ ($k\in S_-$ can be bounded similarly). The proof is similar to the proof in Lemma~\ref{lem: IH stage 2.1}, we omit for simplicity. 

    We now consider $\norm{\vw_{e_+}^\supt{t+1}}_\infty$ and $\norm{\vu_{e_-}^\supt{t+1}}_\infty$.   For $k\not\in S$, consider $w_k$ ($u_k$ can be bounded similarly)
    \begin{align*}
        |w_k^\supt{t+1}| 
        \le& \left(1 + 2\lambda\eta\norm{\mX^\top\vr^\supt{t}/n}_\infty\right) w_k^\supt{t}\\
        \le& \prod_{T_{21}\le\tau\le t}\left(1+ (1-\Omega(\eta d/n))^{\tau-T_{21}} O(\eta\lambda\sigma/\sqrt{n})\right) O(\alpha)\\
        \le& \left(1+\sum_{T_{21}\le\tau \le t} (1-\Omega(\eta d/n))^{\tau-T_{21}} O(\eta\lambda\sigma/\sqrt{n})\right) O(\alpha)\\
        \le& O(\alpha+\lambda\sigma\sqrt{n}\alpha/d) = O(\alpha),
    \end{align*}
    where in the second to last line we use the fact that $\prod_i (1+q_i)=e^{\sum_i\ln(1+q_i)}\le e^{\sum_i q_i}\le 1+O(\sum_i q_i)$ for $\sum_i q_i=O(1)$.

    It remains to consider $w_k$ with $k\in S_-$ and $u_k$ with $k\in S_+$. We will focus on $w_k$ with $k\in S_-$, the other follows the same calculation. Similar in the proof of Lemma~\ref{lem: IH stage 1}, we have
    \begin{align*}
        w_k^\supt{t+1} u_k^\supt{t+1}
        =& \left(1-2\eta\lambda\left(\frac{1}{n}\mX^\top \vr^\supt{t}\right)_k\right)w_k^\supt{t} 
        \cdot \left(1+2\eta\lambda\left(\frac{1}{n}\mX^\top \vr^\supt{t}\right)_k\right)u_k^\supt{t}
        \le w_k^\supt{t} u_k^\supt{t}\le \alpha^2.
    \end{align*}
    We know $u_k^\supt{t}\ge\alpha$. This implies that $|w_k^\supt{t}|\le \alpha$.
    
    For $\norm{\vr}_2$, we can bound it the same as in the proof of Lemma~\ref{lem: IH stage 1}.

    For $\norm{\vv_S}$, we have by Lemma~\ref{lem: stage 2.1}
    \begin{align*}
        \norm{\vv_S^\supt{t+1}}_2
        \le&\norm{\vv_S^\supt{t}}_2 + \eta \norm{\frac{1}{n}(\mX^\top\vr^\supt{t})_S}_2  
        \le \norm{\vv_S^\supt{T_1}}_2 + \sum_{T_{21}\le\tau\le t} (1-\Omega(\eta d/n))^{\tau-T_{21}}O(\eta\sigma/\sqrt{n})\\
        =&O(\sqrt{s}(n/d)\log^2(d)(B_\xi+\sigma\sqrt{n/d}))
    \end{align*}

    For $\norm{\vv}_2$, we have
    \begin{align*}
        \norm{\vv^\supt{t+1}-\vv^\supt{T_{21}}}_2
        \le \eta\sum_{T_{21}\le \tau \le t} \norm{\frac{1}{n}\mX^\top \vr^\supt{\tau}}_2
        \le \sum_{T_{21}\le\tau\le t} (1-\Omega(\eta d/n))^{\tau-T_{21}}O(\eta\sigma/\sqrt{n})
        =O(\sigma\sqrt{n}/ d).
    \end{align*}
    Note that $\norm{\vv^\supt{T_{21}}}_2=O(\sigma\sqrt{n/d})$, thus we have $\norm{\vv^\supt{t+1}}_2=O(\sigma\sqrt{n/d})$.
    
    In this way, we finish the induction proof. It remains to bound $T_{22}-T_{21}$. Given $\norm{\vr^\supt{t}}_2=(1-\Omega(\eta d/n))^t O(\sigma)$, we know $T_{22}-T_{21}=O(n\log (\sigma/\eps)/\eta d)$. Moreover, we can see in the above proof that it will still hold after $T_{22}$, thanks to the geometric decreasing of $\norm{\vr}_2$.

\end{proof}

\section{Proof of main result Theorem~\ref{thm: main}}
In this section, we give the proof of main result. Given that we have already characterized the training dynamics to the convergence in Stage 1 and Stage 2, it immediately follows from the results for Stage 1 (Lemma~\ref{lem: stage 1}) and Stage 2 (Lemma~\ref{lem: stage 2}).
\thmmain*
\begin{proof}
    First note that Lemma~\ref{lem: stage 2} follows from the Lemma~\ref{lem: stage 2.1} for Stage 2.1 and Lemma~\ref{lem: stage 2.2} for Stage 2.2. Then, it suffices to combine Lemma~\ref{lem: stage 1} and Lemma~\ref{lem: stage 2} in Section~\ref{sec: pf sketch}, since
    \begin{align*}
        \norm{\vbeta-\vbeta^*}_2
        &\le \norm{\frac{d}{n}\vv_S+\lambda\vw_{S_+}^{\odot2}-\lambda\vu_{S_-}^{\odot2}-\vbeta^*}_2
        + \norm{\frac{d}{n}\vv_S}_2
        + \norm{\vv}_2
        + \norm{\lambda\vw_{e_+}^{\odot2}-\lambda\vu_{e_-}^{\odot2}}_2\\
        &= O\left(\sqrt{s}\log(n)\log(n/\alpha)\left(\sigma\sqrt{\frac{\log(d)}{n}}+\sigma\sqrt{\frac{n}{d}}\right)\right)
    \end{align*}
\end{proof}

\section{Synthetic Experiments}

\begin{figure}[th]
    \centering
    \includegraphics[width=0.7\linewidth]{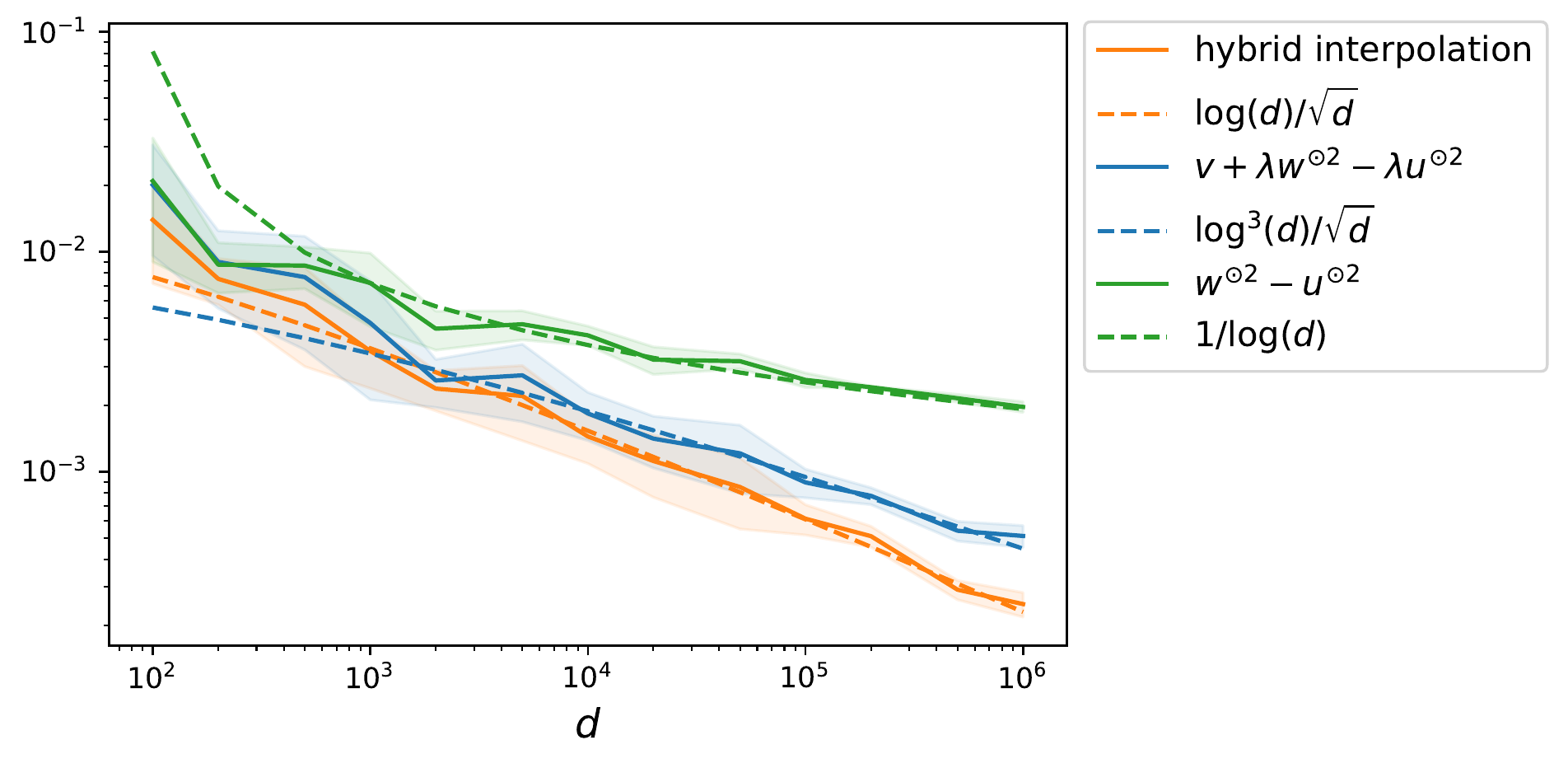}
    \caption{Test loss vs. dimension $d$ when fixing the ratio $d/n=\sqrt{d}/3$ for 3 different interpolating method: hybrid interpolation with Lasso \citep{muthukumar2020harmless}, model $\vv+\lambda\vw^{\odot2}-\lambda\vu^{\odot2}$ as we focused in the paper and model $\vw^{\odot2}-\vu^{\odot2}$ that only keeps the second order term. Solid lines represent the mean and shaded regions represent the standard deviation of test loss during 3 experiments. Dashed lines represent the corresponding order.}
    \label{fig:scaling}
\end{figure}

In this section, we run synthetic experiments to verify our theoretical results. We choose $d$ from 100 to $10^6$ and set $n=3\sqrt{d}$. The target $\vbeta^*=(1/\sqrt{3},-1/\sqrt{3},1/\sqrt{3},0,\ldots,0)^\top$, data $\vx_i\sim N(\vzero,\mI)$ sampled from Gaussian distribution and noise level $\sigma=0.1$. We compare 3 different interpolation method:
\begin{itemize}
    \item hybrid interpolation \citep{muthukumar2020harmless}: As a 2-step procedure, we first use Lasso (implemented in \texttt{sklearn}) with $\ell_1$ regularization coefficient on the order of $\Theta(\sigma\sqrt{\log(d)/n})$ (Theorems 7.13 and 7.20 in \citet{wainwright2019high}). We choose the coefficient with the best test loss among the choice of $\{1/10, 1/5, 1/2, 1, 2, 5, 10\}*\sigma\sqrt{\log(d)/n}$. In the second step, we use the min-$\ell_2$-norm interpolator to fit the residual.

    \item Model $\vv+\lambda\vw^{\odot2}-\lambda\vu^{\odot2}$: As suggested in our main result, we initialize $\vv=0$ and $\vw=\vu=\alpha\vone$ with $\alpha=10^{-10}$. We set $\lambda = 100 d/\sigma n\log(n)(\sqrt{\log(d)/n}+\sqrt{n/d})$ and run gradient descent with stepsize $\eta=10^{-6}$ until training loss reaches $10^{-4}$.

    \item Model $\vw^{\odot2}-\vu^{\odot2}$: We use small initialization that sets $\vw=\vu=\alpha\vone$ with $\alpha=10^{-15}$. We run gradient descent with stepsize $\eta=10^{-6}$ until training loss reaches $10^{-4}$.
\end{itemize}

Our results are shown in Figure~\ref{fig:scaling}. We can see that with fixed ratio $d/n=\sqrt{d}/3$, as $d$ increases, the test loss of different method decreases with different rate. The hybrid interpolation gives the smallest test loss and our learner model $\vv+\lambda\vw^{\odot2}-\lambda\vu^{\odot2}$ gives a similar performance. This agrees with what our theoretical result suggests. The model $\vw^{\odot2}-\vu^{\odot2}$ that only uses second-order term performs worse than others. This is expected as we know such parametrization with small initialization converges to min-$\ell_1$-norm interpolator \citep{woodworth2020kernel}, and min-$\ell_1$-norm interpolator gives large test loss $\Omega(\sigma^2/\log(d/n))$ in the sparse noisy regression setting \citep{chatterji2022foolish, wang2022tight}.

\end{document}